\documentclass{article}

\usepackage[accepted]{aistats2024}
\usepackage[round]{natbib}
\usepackage{balance}

\usepackage[utf8]{inputenc} 
\usepackage[T1]{fontenc}    
\usepackage{hyperref}       
\usepackage{url}            
\usepackage{booktabs,array,longtable}       
\usepackage{amsfonts}    
\usepackage{nicefrac}       
\usepackage{microtype}
\usepackage{soul}
\usepackage{algorithm}
\usepackage{algpseudocode}
\usepackage{subcaption}
\usepackage[table,xcdraw]{xcolor}
\usepackage{amsmath,amssymb}
\usepackage{amsthm,enumitem}
\usepackage{mathtools}
\usepackage{mathabx}

\newcommand{\hltodo}[1]{{\bf \hl{[#1]}}}

\newcommand{\R}{\mathbb{R}}

\newcommand{\confidenceset}{\mathcal{C}_t}
\newcommand{\bmu}{\boldsymbol{\mu}}
\newcommand{\blambda}{\boldsymbol{\lambda}}
\newcommand{\bw}{\boldsymbol{w}}
\newcommand{\pol}{\boldsymbol{\pi}}

\newcommand{\E}{\mathbb{E}}
\newcommand{\F}{\mathcal{F}}
\newcommand{\event}{\mathcal{E}}

\newcommand{\KL}[2]{d(#1, #2)}
\newcommand{\binaryKL}[2]{\mathbf{kl}(#1 || #2)}
\newcommand{\Alt}{\Lambda_\F}
\newcommand{\Cone}{\mathcal{N}}
\newcommand{\Domain}{\mathcal{D}}
\newcommand{\Assumption}[1]{{\bf Assumption #1:}}
\newcommand{\Scenario}[1]{{\bf Scenario #1:}}
\newcommand{\allocationset}{\Pi}

\newcommand{\chinf}{\text{ch}}

\newcommand{\optpol}{\pol^*}
\newcommand{\neighbors}{\mathcal{V}_{\F}(\optpol)}
\newcommand{\chartime}{T_\F}

\DeclareMathOperator*{\argmax}{arg\,max}
\DeclareMathOperator*{\argmin}{arg\,min}

\newtheorem{theorem}{Theorem}

\newtheorem{lemma}{Lemma}
\newtheorem{corollary}{Corollary}
\newtheorem{definition}{Definition}
\newtheorem{example}{Example}
\allowdisplaybreaks

\begin{document}

\runningauthor{Emil Carlsson, Debabrota Basu, Fredrik D. Johansson, Devdatt Dubhashi}

\twocolumn[

\aistatstitle{Pure Exploration in Bandits with Linear Constraints}

\aistatsauthor{Emil Carlsson$^1$, Debabrota Basu$^2$, Fredrik D. Johansson$^1$, Devdatt Dubhashi$^1$}

\aistatsaddress{1. Department of Computer Science and Engineering, Chalmers University of Technology, Sweden. \\
 2. Équipe Scool, Univ. Lille, Inria, CNRS, Centrale Lille, UMR 9189 - CRIStAL, France.} ]

\begin{abstract}
We address the problem of identifying the optimal policy with a fixed confidence level in a multi-armed bandit setup, when \emph{the arms are subject to linear constraints}. Unlike the standard best-arm identification problem which is well studied, the optimal policy in this case may not be deterministic and could mix between several arms. This changes the geometry of the problem which we characterize via an information-theoretic lower bound. We introduce two asymptotically optimal algorithms for this setting, one based on the Track-and-Stop method and the other based on a game-theoretic approach. Both these algorithms try to track an optimal allocation based on the lower bound and computed by a weighted projection onto the boundary of a normal cone. Finally, we provide empirical results that validate our bounds and visualize how constraints change the hardness of the problem. 
\end{abstract}

\section{Introduction}

 A classical problem in the multi-armed bandit framework is \emph{pure exploration}~\citep{lattimore_szepesvari_2020}, where the task of a learner is to answer some query about a set of actions, also known as arms, by iteratively choosing between the actions and receiving an immediate reward sampled from a distribution associated with the action. A very well-studied problem in this context is Best-Arm Identification (BAI), where a learner is trying to identify the arm with the highest expected reward~\citep{EvenDar2002PACBF, bubeck2009pure, Kalyanakrishnan12}. The BAI problem has many applications such as hyper-parameter tuning~\citep{li2017hyperband}, clinical trials~\citep{aziz2021multi}, communication networks~\citep{lindstaahl2022measurement} and user studies~\citep{losada2022day}. However, many real-world scenarios often involve \emph{constraints on the arms} that must be satisfied. For example, in recommender systems, one may need to ensure diversity and genre constraints~\citep{Kunaver17}, or fairness of exposure~\citep{wang2021fairness}. In clinical trials, one may need to account for toxicity constraints of the available treatments~\citep{brannath2009confirmatory,chen2021multi,demirel2022escada}. As a result, standard BAI algorithms are not perfectly fitted in these settings and might have large sample complexity as we show empirically later on in Section~\ref{sec:experiments}. 

In this paper, we introduce the problem of pure exploration in bandits with linear constraints where the goal is to identify, with a fixed confidence, a policy that maximizes the expected rewards over arms while satisfying some given constraints. A set of constraints may change the nature of the pure exploration problem fundamentally. In particular, the optimal policy \emph{may not be deterministic}, and \textit{finding the best arm may not be sufficient}. Let us consider the following example.
\begin{figure}[h]
    \centering
    \includegraphics[width=0.48\textwidth]{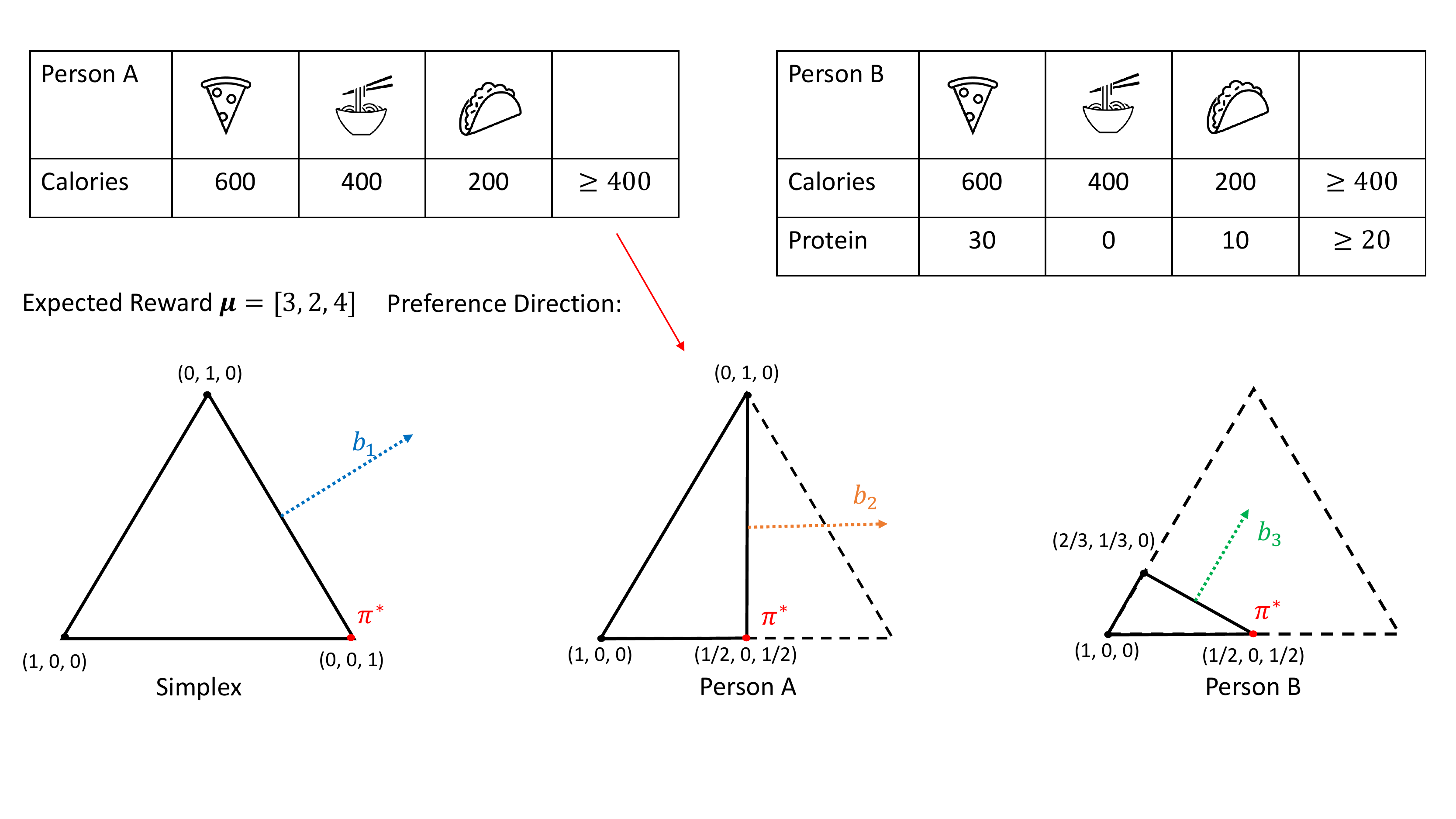}
    \caption{A Visual Representation of Example~\ref{ex:1}. Left figure with the full simplex represents the unconstrained problem. While the constraints of person A (middle) and person B (right) modify the problem to be harder and easier than the unconstrained one.}\label{fig:intro_fig}
\end{figure}

\begin{example}[Optimal meal plan]\label{ex:1}
    Two people, A and B, are searching for a meal plan $\pol$ that maximizes taste, i.e. expected reward $\bmu^\top \pol$, while satisfying some nutrition constraints. Without any constraints this setting reduces to BAI and can be viewed as searching for the optimal policy over the probability simplex. However, as illustrated in Figure~\ref{fig:intro_fig}, the nutrition constraints alter the set of feasible sets and a person might have to mix between several dishes to satisfy the constraints while maximizing the reward. In Figure~\ref{fig:intro_fig}, the red arrow indicates the preference direction and the red dot corresponds to the optimal policy for each case. The dotted arrows, $\boldsymbol{b}_i$, corresponds to the normal of that boundary, i.e. the constraint causing the boundary, and as we will see later, in Figure~\ref{fig:projections}, the distance between $\bmu$ and $\boldsymbol{b}_i$ controls the hardness of the problem. For person A, the distance between $\boldsymbol{b}_2$ and $\bmu$ decreases compared to the unconstrained case, while it increases for person B. Thus, the problem of finding the optimal pure exploration policy gets easier for person B while harder for person A. This is quantified by the minimum number of samples required to identify the optimal policies for person A, B, and the unconstrained case (ref. Fig.~\ref{fig:projections}).
\end{example}

As illustrated in Figure~\ref{fig:intro_fig}, a learner may need to search for a \emph{stochastic policy} that allocates positive probabilities to multiple arms and this influences how an efficient learner should explore. Depending on the constraints, the learner’s task may become easier or harder, e.g. because the learner may need to explore several arms more extensively, or the constraints may remove several near-optimal policies, which makes the problem easier. These observations yield the following fundamental questions:

\textit{How do a specific set of constraints impact a pure exploration problem in terms of the minimum number samples required to identify the optimal policy?}

\textbf{Our Contributions.}  We define the problem of pure exploration in bandits with linear constraints and derive a corresponding lower bound on the sample complexity of any algorithm. We further derive an explicit lower bound for arms corresponding to Gaussian distributions, which shows that the hardness depends on the projection of $\bmu$ onto boundary of a normal cone, and that the lower bound diminishes with the increasing condition number of the constraints defining the optimal policy. Our results show that the lower bound can be thought of as a zero-sum game where the learner plays an exploration strategy and the adversary plays a constraint that is not active at the optimal policy. These insights allow us to modify the standard BAI algorithms, such as Track-and-Stop~\citep{garivier2016optimal} and the game-theoretic algorithm~\citep{Degenne}, and extend them to the constraint setting. We prove that our proposed algorithms are optimal in the asymptotic regime for the pure exploration problem with known linear constraints. Finally, we empirically evaluate the algorithms, both on synthetic and realistic data.

\subsection{Related work}
Now, we review some works on policy learning, a classical problem in decision-making~\citep{bechhofer1958sequential}, that deal with known or learned constraints on decisions and/or constraint exploration due to safety, fairness, or other preferences.


\textbf{Adapting to known constraints.} 
Constraints are often used to ensure safety in reinforcement learning, online learning and control~\citep{moldovan2012safe,gillulay2011guaranteed, Wan22, vayasz22}. In the bandit literature, some variants of the best-arm identification (BAI) problem impose constraints on the chosen arm, or on the exploration process. \citet{wang2022best, Camilleri22} studies the setting with unknown linear rewards under known safety constraints but only allow single coordinate actions. \citet{faizal2022constrained} consider BAI under fixed budget with known constraints on the arms. Their setting differs from ours in that \textit{we look for a best “policy” over arms with linear constraints rather than a single best arm}.

\textbf{Learning unknown constraints.} 
\citet{sui2015safe,sui2018stagewise} study online optimization of an unknown function $f$ with constraints on $f$, but without formal analysis. In the bandit literature, constraints are mostly studied in the regret-minimization setting. \cite{moradipari2021safe} and \cite{pacchiano2021stochastic} consider regret minimization in linear bandits under linear constraints from Bayesian and Frequentist perspectives, respectively. \citet{amani2019linear} study regret minimization in linear contextual bandits with unknown and unobserved linear constraints. \citet{wang2021fairness}  aims to minimize the fairness regret to ensure proportional exposure for each arm, which implies a known structure for the policies. Unlike these works, we focus on the pure exploration setting. \citet{lindner2022interactively} considers constrained linear best-arm identification arm are vectors with \emph{known} rewards and a single \emph{unknown} constraint (representing preferences) on the actions.

\textbf{Pure exploration algorithms.}
Our Constrained Track-and-Stop algorithm, (CTnS, Section~\ref{sec:tns}), follows the Track-and-Stop TnS) meta-scheme proposed by \citet{garivier2016optimal}. In TnS, one tracks an optimal allocation with respect to a lower bound and assumes that the current estimate is the true environment. This approach has been applied to various bandits, e.g., linear bandits~\citep{jedra2020optimal}, spectral bandits~\citep{Kocak21}, heavy-tailed bandits~\citep{agrawal2020optimal}, bandits with multiple correct answers~\citep{Degenne19}, and latent bandits~\citep{mwai2023fast}. The Constrained Game Explorer, (CGE, Section~\ref{sec:cge}), follows the gamification approach to pure-exploration, which treats the lower bound as a zero-sum game between an allocation player and instance player. This approach was first introduced by \citet{Degenne}, and later used for best-arm identification in linear bandits~\citep{degenne2020gamification} and combinatorial bandits~\citep{degenne2020gamification}. In particular, CGE is an extension of the sampling rule of~\citep{Degenne} to the case of known linear constraints.

\textbf{Transductive linear bandit.} Another related setup is the transductive linear bandit~\citep{fiez2019sequential}, where one set of arms, $\mathcal{A}$, are played during exploration while the goal is to detect the best arm in some other known set, $\mathcal{Z}$. This is related to our setting since we want to learn the best policy but only have access to arms. Hence, our model can be viewed as a natural special case of the transductive linear bandit where $\mathcal{A}$ is the standard basis and $\mathcal{Z}$ is the set of policies. However, the existing literature on transductive bandits does not study the impact of linear constraints that we explicitly study here and the resulting algorithms are different.

\textbf{Bandits with knapsacks.} Our work is also related to the bandit with knapsack~\citep{Badanidiyuru18, agrawal2016linear, immorlica2022adversarial}. In this model, there are upper bounds on the total amount of resources a learner can consume while interacting with the bandit and each arm has its own resource consumption. The goal is to minimize the cumulative regret and the learner has to stop once the resources are depleted. This is different from our setting since we consider the problem of finding the best policy and not regret minimization. Our constraints are also not budget constraints but constraints in the policy space.

\section{Problem formulation}

We consider a multi-armed bandit problem with $K$ arms that corresponds to reward distributions, $\{P_a\}_{a=1}^K$, with unknown means $\{\mu_a\}_{a=1}^K$ and support $\R$. At each time step $t$, a learner chooses to play one of the arms, $A_t \in [K]$, and observes an immediate reward $R_t$, drawn from the reward distribution $P_{A_t}$. The learner has access to a non-empty and compact set of feasible policies
\begin{align}\label{eq:constraint_set}
    \F \triangleq \left \{ \pol \in \Delta_{K-1}: B\pol \leq c \right\},
\end{align}
where $\Delta_{K-1}$ is the $K$-simplex and $B \in \R^{N \times K}$ and $\boldsymbol{c} \in \R^{N}$, are known parameters of the linear constraints. For the ease of the presentation, we absorb the simplex constraints in $B$ and $\boldsymbol{c}$. Hereafter, these variables refer to both the simplex constraints, and the additional linear constraints of the problem.
\textit{The goal of the learner is to recommend, with probability at least $1-\delta$, the \emph{unique} optimal policy} $\pol^*_{\bmu, \F}$ satisfying \begin{equation}\label{eq:lp}
\begin{aligned}
& \pol^*_{\bmu, \F} \triangleq \underset{\pol \in \F}{\text{arg max}}
& & \bmu^\top \pol.
\end{aligned}
\end{equation}
 When it is clear from the context, we denote $\pol_{\bmu, \F}^*$ as $\optpol$. We refer to such a learner as a $\delta$\emph{-PAC learner}. As $1-\delta$ quantifies the correctness of the learner, we also want it to be efficient, i.e. to detect the optimal policy \textit{fast}. 
Let $\tau_{\delta}$ denote the random \textit{stopping time} at which the learner stops interacting with the bandit and makes a recommendation with confidence $1-\delta$.
\textit{We aim to design a $\delta$-PAC learner that minimizes the expected stopping time $\E[\tau_\delta]$, a.k.a. sample complexity, needed to find the optimal policy.} 

Depending on the application, a learner can abide by the constraints of Equation~\eqref{eq:constraint_set} in two ways:
\begin{itemize}[leftmargin=*]
    \item  \Scenario{1: End-of-time constraint} The learner does not have to take the constraints into account during exploration. Only the final recommended policy needs to satisfy the constraints.
    \item  \Scenario{2: Anytime constraint} The exploration policy needs to satisfy the constraints \emph{in expectation} during exploration, i.e. the exploration policy $\bw_t$ needs to satisfy $\bw_t \in \F$.
\end{itemize}
For example, Scenario 1 arises while using a more sophisticated hardware to search for an optimal policy, that should satisfy some energy-constraints, before deploying it on a low-energy hardware. In contrast, Scenario 2 can be thought of as performing the search directly on the low-energy hardware. 
Now, we explicitly state the assumptions used in this study:
\begin{itemize}[leftmargin=*]
    \item \Assumption{1} The reward of each arm $i \in [K]$ is distributed according to a sub-Gaussian single-parameter exponential family parameterized by its unknown mean $\mu_i$.
    \item \Assumption{2} The vector of arm means, $\bmu$, lies in a bounded domain  $\Domain = [ \mu_{\min}, \mu_{\max}]^K$.
    \item \Assumption{3} The optimal solution $\pol^*_{\bmu, \F}$ to the linear program in Equation~\eqref{eq:lp} is unique. 
\end{itemize}
Assumptions 1 and 2 are standard in the literature~\citep{Degenne19,degenne2020gamification}. Assumption 3 is the analogue of assuming a unique best arm in the BAI problem, and it ensures that the optimum of Equation~\eqref{eq:lp} is an extreme point. Hence, the optimal policy $\pol^*_{\bmu, \F}$ always corresponds to an extreme point in the polytope $\F$. In Appendix~\ref{app:epsilon_good}, we discuss about relaxation to $\epsilon$-good policies.

\noindent\textbf{Notations.} Let $\allocationset$ denote the set of feasible exploration policies. Thus, for Scenario 1, $\allocationset = \Delta_{K-1}$, and  $\allocationset = \F$ for Scenario 2. We denote the KL-divergence between two single-parameter exponential family distributions with mean $x$ and $y$ as $\KL{x}{y}$. Additionally, if the random variables are Bernoulli, we denote the KL-divergence as $\binaryKL{x}{y}$.

\section{Lower bound}

Lower bounds on the sample complexity of a $\delta$-correct algorithm, i.e. $\E[\tau_\delta]$, is a driving force in designing good algorithms in the BAI literature~\citep{garivier2016optimal,Degenne19,agrawal2020optimal}.

Given a problem instance $\bmu$, a learner needs to collect enough information about the problem to be able to rule out all alternative instances, $\blambda$, for which we have $\max_{\pol \in \F} \blambda^\top \pol > \blambda^\top \optpol$ with confidence at least $1-\delta$ . We refer to this set of instances as the Alt-set and denote it as 
\begin{align}
    \Alt(\bmu) \triangleq \lbrace\blambda \in \Domain: \max\limits_{\pol \in \F} \blambda^\top \pol > \blambda^\top \optpol \rbrace.
\end{align}
\citet{garivier2016optimal} introduced general techniques for deriving lower bounds on the sample complexity of any $\delta$-PAC learner, which depends on the the distance from $\bmu$ to the closest $\blambda \in \Alt(\bmu)$ in an information-theoretic sense. 

\begin{figure}[h]
    \centering
    \includegraphics[width=0.45\textwidth]{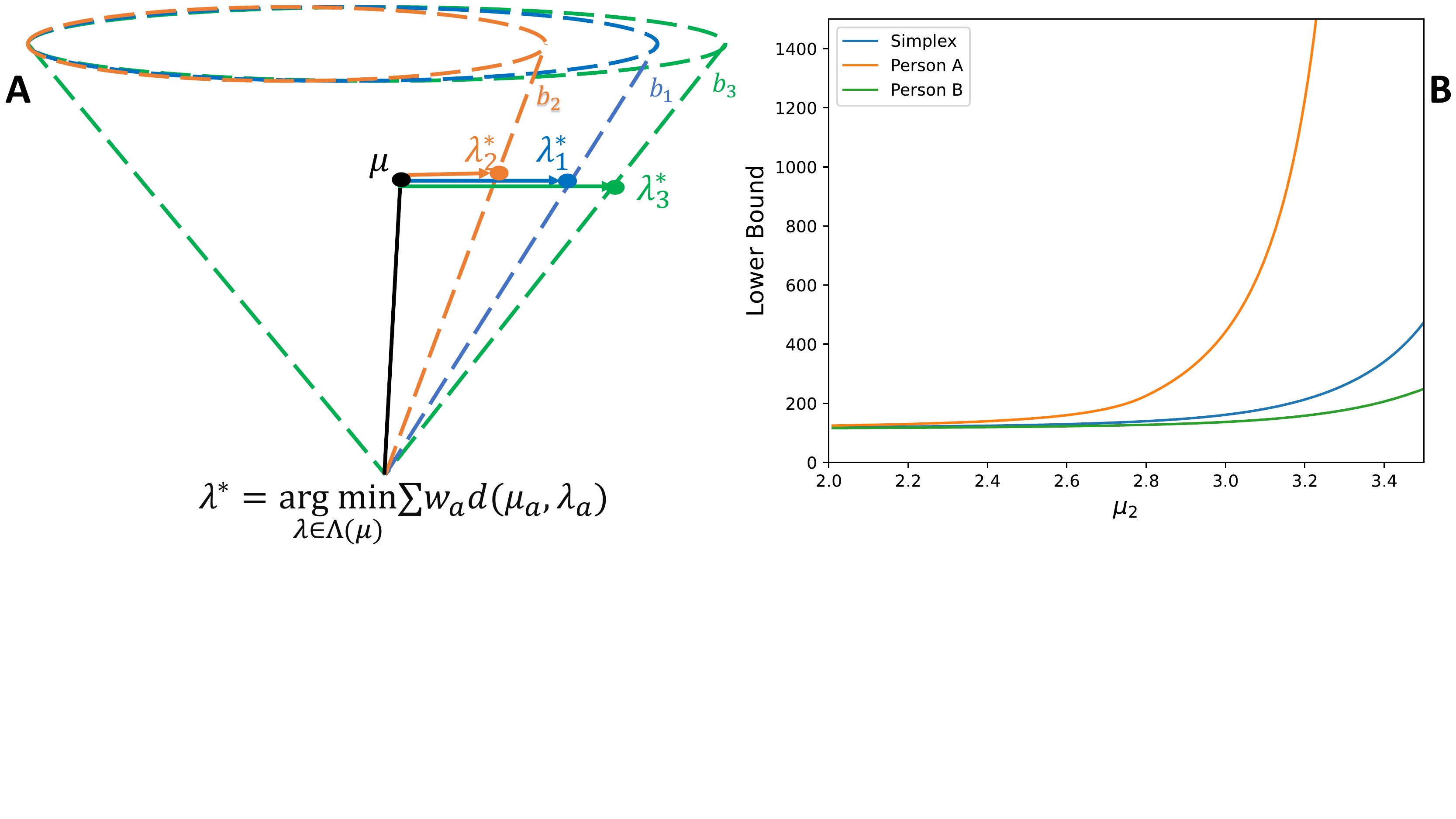}
    \caption{Computing the $\blambda$ satisfying Equation~\ref{eq:proj}, i.e. the \emph{most confusing instance}, can be viewed as an information-theoretic projection onto the boundary of the normal cone spanned by the active constraints at $\pol_{\bmu}$. In A) we see the different normal cones for the three different examples in Figure~\ref{fig:intro_fig}. In B) we have fixed $\mu_1$ and $\mu_3$, as in Figure~\ref{fig:intro_fig}, and plot the lower bound, assuming $N(0, 1)$ noise and with $\delta=0.1$, for increasing $\mu_2$ which mean that we are moving $\bmu$ closer to the  boundaries in A). We observe an inverse relationship between the distance to the boundary and the lower bound, properly characterized in Corollary~\ref{prop:lb_lower_bound}. }\label{fig:projections}
\end{figure}

We extend these general proof techniques and show that the expected stopping time of any $\delta$-PAC algorithm $\phi$ for BAI with linear constraints satisfies 

\begin{align}\label{eq:lb}
     \E_{\bmu, \phi} \left[ \tau_{\delta} \right] \geq \chartime(\bmu) \binaryKL{\delta}{1-\delta}.
\end{align}
where $\chartime(\bmu)$ is the \emph{characteristic time}, defined as 
\begin{align}\label{eq:char_time}
    \chartime^{-1}(\bmu) =\sup_{\bw \in \allocationset}\inf_{\blambda \in \Alt(\bmu)} \sum_{a=1}^K w_a \KL{\mu_a}{\lambda_a}.
\end{align}
The supremum in Equation~\eqref{eq:char_time} hints towards the existence of some optimal exploration policy $\bw$, which any optimal algorithm should try to track. This is exactly the idea behind the Track-and-Stop meta-scheme~\citep{garivier2016optimal} (details in Section~\ref{sec:tns}). In order to design algorithms achieving the lower bound in Equation~\eqref{eq:lb}, we need to solve the optimization problem in Equation~\eqref{eq:char_time}. This requires a more explicit characterization of $\Alt(\mu)$, continuity properties of the function $D(\bw, \bmu, \F) \triangleq \inf_{\blambda \in \Alt(\bmu)} \sum_{a=1}^K w_a \KL{\mu_a}{\lambda_a}$, and the set of optimal allocations $w^*(\bmu)$.

To derive an explicit expression for $\Alt(\mu)$, let $M$ be the number of active constraints for $\optpol$, $B_{\optpol} \in \R^{M \times K}$ be a submatrix of $B$ consisting of all these active constraints,  and $\boldsymbol{c}_{\optpol} \in \R^M$ the corresponding bounds in $\boldsymbol{c}$. Hence, there exists \emph{at least} $K$ linearly independent rows in $B_{\optpol}$, i.e. a matrix $\hat{B}_{\optpol} \in \R^{K \times K}$ and vector $\hat{\boldsymbol{c}}_{\optpol} \in \R^K$, such that $\optpol = \hat{B}_{\optpol}^{-1}\hat{\boldsymbol{c}}_{\optpol}$. Since our objective (Equation~\eqref{eq:lp}) is a linear program, we can leverage the optimality condition stating that  $\bmu$ must be in the normal cone of the optimal solution~\citep{Boyd04}. Hence, we express the Alt-set as 
$    \Alt(\bmu) = \left\{\blambda: \blambda \notin \Cone(\optpol) \right\}.$
Here, $\Cone(\optpol) := \left \{\blambda: \blambda = B_{\optpol}^\top \boldsymbol{v}, \boldsymbol{v} \in \R_{\geq 0}^M \right \}$ is the normal cone spanned by the active constraints for $\optpol$. 

Further, we say that $\pol'$ is a neighbor of $\optpol$ if it is an extreme point in $\F$ and shares $K-1$ active constraints with $\optpol$. We denote the set of all neighbors of $\optpol$ as $\neighbors$.
Hence, we can decompose the Alt-set into a union of a finite number of half-spaces 
$
    \Alt(\bmu) = \bigcup_{\pol' \in \neighbors} \left \{ \blambda: \blambda^\top (\optpol - \pol') < 0 \right \}.$
This formulation implies that if $\optpol$ is not an optimal policy for the instance $\blambda$, there must exist an direction for the simplex algorithm to follow to increase the expected reward, i.e. $\exists \pol' \in \neighbors:$ $\blambda^\top (\optpol - \pol') < 0$. This formulation of Alt-sets lead us to the observation that the most confusing instances in the Alt-set w.r.t. $\bmu$ lay on the boundary of the normal cone.

Specifically, Lemma~\ref{lm:proj} shows that the function $D(\bw, \bmu, \F)$ is a weighted projection onto the plane $\blambda^\top(\pol' - \optpol) = 0$ for some $\pol' \in \neighbors$, as shown in Figure~\ref{fig:projections}.

\begin{lemma}[Projection lemma]\label{lm:proj}
For any $\bw \in \allocationset$ and $\bmu$ it holds that
\begin{align}\label{eq:proj}
    D(\bw, \bmu, \F) = \min_{\pol' \in \neighbors} \min_{\blambda: \blambda^\top(\optpol - \pol') = 0} \sum_{a=1}^K w_a \KL{\mu_a}{\lambda_a}
\end{align}
\end{lemma}

To compute $D(\bw, \bmu, \F)$ from Equation~\eqref{eq:proj}, we need to have access to the true instance $\bmu$, which we do not have in reality. Rather, we sequentially obtain samples from the arms yielding an estimate $\hat{\bmu}_t$. Thus, we need $D(\bw, \bmu, \F)$ and $\bw^*(\bmu)$ to satisfy continuity properties (Theorem~\ref{lm:properties}) w.r.t $\bmu$, that ensures as the estimates $\hat{\bmu}_t$ converge to $\bmu$, $D(\bw, \hat{\bmu}_t, \F) \to D(\bw, \bmu, \F)$ and our empirical distribution of plays gets closer to some $\bw \in w^*(\mu)$.
\begin{theorem}\label{lm:properties}
Following properties are true for all $\bmu$ and $\F= \left \{ \pol \in \Delta_{K-1}: B \pol \leq c \right\}$ such that the problem $\max_{\pol \in \F} \bmu^\top \pol$ has a unique solution. 
\begin{itemize}[nosep,leftmargin=*]
    \item  The function $(\bw, \bmu) \mapsto D(\bw, \bmu, \F)$ is continuous.
    \item The function $\bmu \mapsto \chartime(\bmu)$ is continuous. 
    \item The set-valued function $\bmu \mapsto  w^*(\bmu)$ is upper hemicontinuous (definition in Appendix~\ref{app:useful}). 
    \item The set $ w^*(\bmu)$ is convex. 
\end{itemize}
\end{theorem}
\subsection{Lower bound for Gaussian distributions}
To gain further insights on how the constraints alter the lower bound in Equation~\eqref{eq:lb}, we consider the special case where all arms are Gaussian distributions with equal variance $\sigma^2$. 
This leads us to a close-form of the projection in Lemma~\ref{lm:proj} as in Theorem~\ref{cor:projection}.
\begin{theorem}
\label{cor:projection}
If the arms follow Gaussian distributions with identical variance $\sigma^2$ and $w_a > 0$ $\forall a$, we have that the projection $\min_{\blambda \in \Domain: \blambda^\top(\optpol - \pol') \leq 0} \sum_{a=1}^K w_a \KL{\mu_a}{\lambda_a}$ for any $\pol' \in \mathcal{V}_{\F}({\optpol})$ is satisfied by $
    \lambda_{a, \pol'} = \mu_a - \gamma \frac{(\optpol - \pol')_a}{w_a}, $
for $\gamma = \frac{\bmu^\top \left(\optpol - \pol' \right) }{\sum_a \frac{(\optpol - \pol')^2}{w_a}}$, and the characteristic time is
\begin{align*}
   \chartime(\bmu)^{-1} &= \max_{\bw \in \allocationset} \min_{\pol' \in \neighbors} \frac{1}{2\sigma^2} \frac{\left(\bmu^\top \left(\optpol - \pol' \right) \right)^2}{\sum_{a} \frac{1}{w_a}(\optpol - \pol')^2_a} \\
   &= \max_{\bw \in \allocationset} \min_{\pol' \in \neighbors} \frac{1}{2\sigma^2} \frac{ \|\optpol - \pol'\|^2_{\bmu\bmu^{\top}}}{ \|\optpol - \pol'\|^2_{\mathrm{Diag}(1/w_a)}}
\end{align*}
Here, $\mathrm{Diag}(1/w_a)$ is a diagonal matrix with $a$-th entry of the diagonal as $1/w_a$.
\end{theorem}

In the classical BAI setting, i.e. we only have simplex constraints, the expressions in Theorem~\ref{cor:projection} reduces to the BAI results of \citet{Kaufmann16}, see Appendix~\ref{app:lb} for a derivation. From Theorem~\ref{cor:projection}, we further derive a lower and an upper bound on the characteristic time. Let us define $d_{\pol'} \triangleq \min_{\blambda: \blambda^\top(\optpol - \pol')=0} ||\bmu - \blambda||_2$ and note that this is the distance between $\bmu$ and the hyperplane $\optpol - \pol'=0$, see Figure~\ref{fig:projections} for illustration. 
\begin{corollary}
    \label{prop:lb_lower_bound}
The characteristic time $\chartime(\bmu)$ satisfies the following bounds:
\begin{align}
   \min_{\pol' \in \mathcal{V}_{\F}(\optpol)} \frac{2\sigma^2}{d_{\pol'}^2}
  \leq \chartime(\bmu) \leq \min_{\pol' \in \mathcal{V}_{\F}(\optpol)} \frac{2\sigma^2 K}{d_{\pol'}^2}.
\end{align}

\end{corollary}
Corollary~\ref{prop:lb_lower_bound} implies a lower bound of \begin{align*}
    \E[\tau] \geq \min_{\pol' \in \mathcal{V}_{\F}(\optpol)} \frac{2\sigma^2}{d_{\pol'}^2} \binaryKL{\delta}{1-\delta}
\end{align*}

\textbf{Impact of constraints: geometric view.} We first observe that, since the distance-to-projection $d_{\pol'}= \frac{\bmu^\top \left(\optpol - \pol' \right)}{\|\optpol - \pol' \|_2}$, the problem becomes easier when the direction of the reward vector $\bmu$ is aligned with the deviation in policy $\pol^* - \pol'$. Especially, if we only consider deterministic policies, i.e. BAI problem, $d_{\pol'} = \mu_1 - \mu_a = \Delta_a$ where $\mu_1$ is the best arm, $a$ is the arm played by $\pol'$ and we retrieve the lower bound of \citet{Kaufmann16}. 

\textbf{Impact of constraints: constrained optimization view.} We relate the lower bound more explicitly to the constraint matrix $B$ by using the fact that any neighbor $\pol' \in \neighbors$ can be reached from $\optpol$ via an 1-rank update on a matrix $\hat{B}_{\optpol} \in \R^{K \times K}$ consisting of $K$ active constraints at $\optpol$ that are linearly independent. Thus, we only need to change one row in $\hat{B}_{\optpol}$ and one element in the corresponding $\hat{\boldsymbol{c}}_{\optpol}$ to get $B'$ and $\boldsymbol{c}'$ such that $\pol' = B'^{-1}\boldsymbol{c}'$. This results in the lower bound on the sample complexity presented in Corollary~\ref{prop:spectral}.
\begin{corollary}\label{prop:spectral}
For any $\pol' \in \neighbors$, let $\hat{B}_{\pol^*}\in \R^{K \times K}$ be a set of active and linearly independent constraints at $\pol^*$ such that the active constraints at $\pol'$ can be achieved by a one-rank update on $\hat{B}_{\pol^*}$. Let $r'$ be the row in $\hat{B}_{\pol^*}$ that is changed during this one-rank update. 

\textit{Part (a):} Let $\boldsymbol{\Delta} \in \R^K$ denote the vector of the sub-optimality gaps, i.e. $\Delta_a =  \mu_1 - \mu_a$, of each arm, then
\begin{align}\label{eq:slack_diff}
    \chartime(\bmu)^{-1} = \max_{\bw \in \allocationset} \min_{\pol' \in \neighbors} \frac{1}{2\sigma^2} \frac{ \left(\boldsymbol{\Delta}^\top \hat{B}_{\optpol}^{-1}\boldsymbol{e}_{r'} \right)^2}{ \|\hat{B}_{\optpol}^{-1}\boldsymbol{e}_{r'}\|^2_{\mathrm{Diag}(1/w_a)}}
\end{align}
\textit{Part (b):} Let $\kappa^2$ be the condition number of a matrix $\hat{B}_{\optpol} \in \R^{K \times K}$ consisting of $K$ linearly independent active constraints at $\optpol$, then the sample complexity of any $\delta$-PAC learner is lower bounded as 
\begin{align}\label{eq:kappa_bound}
    \E[\tau] = \Omega\left(\frac{H}{\kappa^2} \binaryKL{\delta}{1-\delta}\right)
\end{align}
with $H = \frac{2\sigma^2}{\sum_{a\neq a^{*}} \Delta_a^2}$.
\end{corollary}
Corollary~\ref{prop:spectral} relates constraints, arm sub-optimality, and sample complexity. Equation~\eqref{eq:slack_diff} links sample complexity to perturbations of the optimal policy.  Naturally, if a large perturbation of the optimal policy is only slightly sub-optimal, the sample complexity will be large. In contrast, if a small pertubation is bound to cause the resulting policy to be highly sub-optimal it is easier to detect the optimal policy. Equation~\eqref{eq:slack_diff} also reinterprets the lower bound as a zero-sum game where the agent plays an allocation and an adversary switches an active constraint at $\optpol$ to a non-active one.

Equation~\eqref{eq:kappa_bound} provides a looser bound based on a \textit{suboptimality gap based complexity measure} $H$, and \textit{the condition number}, $\kappa^2$ of the active-constraint matrix, which measures sensitivity of the optimal policy to perturbations. A high $\kappa^2$ implies that small perturbations of the optimal policy will cause a large change of the slack corresponding to the active constraints, making exploration easier. A low $\kappa^2$ means policy perturbations have a smaller impact on the slack making neighboring policies less distinguishable from the optimal one.

\section{Algorithms}\label{ref:algos}

In this section, we focus on extending the classical pure exploration algorithms to the setting of pure exploration with linear constraints.

\textbf{Algorithm design.} We begin by observing that any pure exploration algorithm consists of three components: \textit{a stopping rule}, \textit{a recommendation rule}, and \textit{a sampling strategy}. 
The stopping rule consists of a condition deciding when to halt sampling further. The recommendation rule decides what policy to recommend as the optimal policy. The sampling rule decides which arm to sample next given the history of arms sampled and intermediate policies computed. 

\paragraph{Component 1: Chernoff's stopping rule with constraints.}
As a stopping rule, we extend the Chernoff's stopping rule~\citep{garivier2016optimal}. We first introduce the $\emph{confidence set}$ $\confidenceset(\delta) := \left\{\blambda: \sum_{a=1}^K  N_{a, t} d(\hat{\mu}_{a, t}, \lambda_a) \leq c(t, \delta)\right\}
$,
where $c(t, \delta)$ is a threshold defined in Lemma~\ref{cor:stopping}.
\begin{lemma}[\citet{garivier2016optimal}]\label{cor:stopping}
For any $\alpha > 1$ there exists a constant $C(\alpha, K)$ such that for $c(t, \delta) = \log \frac{t^\alpha C(\alpha, K)}{\delta}$ we have for any $t \in \mathbb{N}$ $
    P\left( \bmu \notin \confidenceset(\delta) \right) \leq \delta.
$
\end{lemma}
Lemma~\ref{cor:stopping} implies that Chernoff's stopping rule is a $\delta$-PAC stopping rule, and we stop when 
\begin{align}\label{eq:stopping}
    \inf_{\blambda \in \Alt(\hat{\bmu}_t)} \sum_{a=1}^K  N_{a, t} d(\hat{\mu}_{a, t}, \lambda_a) > c(t, \delta).
\end{align}
This means that the confidence set is a subset of the normal cone spanned by the active constraints at $\pol_{\hat{\bmu}_t}^*$. The details of the constant in Lemma~\ref{cor:stopping} are deferred to Appendix~\ref{app:upper}. Note that one can also derive a stopping rule via the concentration results of \citet{Kaufmann21}. 

\textbf{Component 2: Recommendation rule.} We recommend the solution of the linear programming (Equation~\eqref{eq:lp}) with the empirical means of the arms at the stopping time, $\pol_{\hat{\bmu}_t}^* = \argmax_{\pol \in \F} \hat{\bmu}_t^\top \pol$. Since the empirical means might not always be within the pre-specified range $\Domain$, we let $\hat{\bmu}_t$ denote the Euclidean projection of the empirical means onto $\Domain$.
\begin{algorithm}[t]
\caption{ Constrained Track-and-Stop (CTnS)}\label{alg:ctns}
\begin{algorithmic}[1]
\Require Confidence level $\delta$, constraints $(B, c)$, exploration set $\allocationset$
\State Play each arm once.
\While{$c(t, \delta) > D(\boldsymbol{N}/t, \hat{\bmu}_t, \F) $} \Comment{{\color{red}Weighted projection via Lemma~\ref{lm:proj}}}
\State Compute $\bw_t^* \in \argmax_{\bw \in \allocationset} D(\bw, \hat{\bmu}_t, \F)$ \Comment{{\color{red} Solve for optimal $\bw$ w.r.t. the constraints} }
\State Play $A_t \in \argmin_a N_{a, t} - \sum_{s=1}^t w_{a, s, \epsilon_s}^{*}$ and observe reward $R_t$
\EndWhile
\State Recommend $\pol_{\hat{\bmu}_t}^* = \argmax_{\pol \in \F} \hat{\bmu}_t^\top \pol$
\end{algorithmic}
\end{algorithm}
\begin{algorithm}[t]
\caption{Constrained Game Explorer (CGE)}\label{alg:gt}
\begin{algorithmic}[1]
\Require Confidence level $\delta$, constraints $(B, c)$, exploration set $\allocationset$

\While{$c(t, \delta) > D(\boldsymbol{N}/t, \hat{\bmu}_t, \F) $} \Comment{{\color{red}Weighted projection via Lemma~\ref{lm:proj}}}
\State Get allocation $\bw_t$ from regret minimizer \Comment{{\color{red} Running Adagrad over $\allocationset$} }
\State  Compute best-response $\blambda_t$ w.r.t. $\bw_t$ and $\hat{\bmu}_t$ \Comment{{\color{red}Weighted projection via Lemma~\ref{lm:proj}}}
\State Compute confidence intervals $\forall a$ $[\alpha_{t, a}, \beta_{t, a} ] = \left\{\xi: N_{a, t} \KL{\hat{\mu}_{a, t}}{\xi} \leq f(t) \right\}$
\State $\forall a$ $U_t^a := \max \left\{\frac{f(t)}{N_{a, t}}, \max_{\xi \in \{\alpha_t^a, \beta_t^a\}} \KL{\xi}{\lambda_{a, t}}\right\}$
\State Update AdaGrad with $l(w_t) = \sum_{a=1}^K w_a U_{a, t} $
\State Play $A_t \in \argmin_a N_{a, t} - \sum_{s=1}^t w_{a, s, \epsilon_s}^{*}$ and observe reward $R_t$
\EndWhile
\State Recommend $\pol_{\hat{\bmu}_t}^* = \argmax_{\pol \in \F} \hat{\bmu}_t^\top \pol$
\end{algorithmic}
\end{algorithm}

\textbf{Component 3a: CTnS.}\label{sec:tns}

First, we present our Constrained Track-and-Stop Algorithm (CTnS, Algorithm~\ref{alg:ctns}), which is an adaptation of the Track-and-Stop (TnS) framework~\citep{garivier2016optimal} to the linear constraint setting with aforementioned stopping and recommendation rules. In Algorithm~\ref{alg:ctns}, we highlight, in {\color{red}red}, the computations that we modify from the original schematic to account for the linear constraints.  The algorithm starts by playing each arm once. Then, until the stopping rule in Equation~\eqref{eq:stopping} fires, it performs \emph{C-tracking}~\citep{garivier2016optimal}. 
This means that we perform a $\max-\min$ oracle call (Line 3), and solve the problem in Equation~\eqref{eq:char_time} w.r.t our current estimate of the means $\hat{\bmu}_t$ to get an optimal allocation $\bw^*_t$. This step leverage our novel projection result in Lemma~\ref{lm:proj}. We track the optimal allocation via $A_t \in \argmin_a N_{a, t} - \sum_{s=1}^t w_{a, s, \epsilon_s}^{*}$,
where $w_{a, t, \epsilon_t}^{*}$ is the projection of $\bw^*_t$ onto $\allocationset \bigcap \{\bw: w_a > \epsilon_t \forall a \}$, and $\epsilon_t = \frac{1}{2 \sqrt{K^2 + t}}$. Note that $\frac{1}{t}\sum_{s=1}^t w_{a, s, \epsilon_s}^{*} \in \allocationset$ due to the convexity of the set of feasible exploration policies/allocations.

\begin{theorem}[Upper bound for CTnS]\label{thm:ctns}
    For any $\alpha > 1$ and $c(t, \delta)$ be defined as in Lemma~\ref{cor:stopping}, we have that the expected stopping time of CTnS satisfies \begin{align*}
        \lim_{\delta \rightarrow 0} \frac{\E[\tau]}{\log \frac{1}{\delta}} \leq \chartime(\bmu), \, \forall \bmu \in \Domain.
    \end{align*}
\end{theorem}
The proof of Theorem~\ref{thm:ctns} can be found in Appendix~\ref{app:ctns} and follows the same structure as the sample complexity proof of the original TnS in \citet{garivier2016optimal}. However, the optimal allocation does not necessarily have to be unique. rather, we use the upper hemicontinuity and convexity of $w^*(\bmu)$, while modifying the tracking lemma originally used by~\citet{garivier2016optimal} with the tracking result of~\citet{Degenne19}. This change allows to track a set of optimal solutions in absence of a unique optimum.

\textbf{Component 3b: CGE.}\label{sec:cge}
Track-and-Stop algorithms, like CTnS, tend to be computationally inefficient for larger problems since they requires a $\max-\min$ call at each iteration. To mitigate this, we adopt the approach of \citet{Degenne}, and treat the optimization problem in Equation~\eqref{eq:char_time} as a two player zero-sum game. This results in the Constrained Game Explorer (CGE), in Algorithm~\ref{alg:gt}.  Instead of solving for an optimal $\bw_t$ at each $t$, as in CTnS, we play one game between an allocation player, who plays $\bw$ to maximize $\sum_{a=1}^K w_a \KL{\hat{\mu}_{a, t}}{\lambda_a}$, and an instance player, who plays the confusing instance $\blambda$ to minimize $\sum_{a=1}^K w_a \KL{\hat{\mu}_{a, t}}{\lambda_a}$. We deploy an instance of AdaGrad~\citep{Duchi11} as the allocation player is taken to be, which enjoys sub-linear regret over any bounded domain when losses are convex, and the instance player is taken to be a best-response w.r.t. the allocation $\bw_t$. The best-response is computed via Lemma~\ref{lm:proj}.
The loss provided to AdaGrad at each time step is $\sum_{a=1}^K w_{a, t} U_{a, t}$, where $U_{a, t}$ induces optimism as $U_{a, t} \triangleq \max_{\xi \in \{\alpha_{a, t}, \beta_{a, t}\}} N_{a, t}\KL{\xi}{\lambda_{a, t}}$. Here, $(\alpha_{a, t}, \beta_{a, t})$ are the endpoints of the confidence interval around $\hat{\mu}_{a, t}$, i.e. $[\alpha_{t, a}, \beta_{t, a} ] = \left\{\xi: N_{a, t} \KL{\mu _{a, t}}{\xi} < f(t) \right\}$, and $f(t) = 3\log t + \log \log t$.
We apply the same tracking as in CTnS.

\begin{theorem}[Upper bound for CGE]\label{thm:cge}
    The expected sample complexity of CGE satisfies
    $\E[\tau] \leq T_0(\delta) + CK,$
    where $T_0(\delta) := \max \left\{t \in \mathbb{N}: t \leq \chartime(\bmu)c(t, \delta) + O(\sqrt{tQ}) + O(\sqrt{t \log t}) \right\}$. $C_\mu$ is problem-dependent constant, C is a universal constant and $Q$ is an upper bound on the losses provided to Adagrad.
\end{theorem}
The full proof of Theorem~\ref{thm:cge} can be found in Appendix~\ref{app:cge}. We simply follow the steps of the proof of Theorem 2 in \citet{Degenne} and apply specifics of our setting when applicable. 

\textit{Theorem~\ref{thm:ctns} and~\ref{thm:cge} show that CTnS and CGE are asymptotically optimally}, i.e. upper bound on their sample complexities match the lower bound of constrained pure exploration for small enough $\delta$.

\section{Experimental analysis}\label{sec:experiments}

We evaluate our algorithms using the threshold $c(t, \delta)= \log \frac{1 + \log \log t}{\delta}$, which is commonly done in the literature ~\citep{garivier2016optimal}, and we set $f(t) = \log t$ in CGE. As benchmarks we will use the lower bound, Equation~\ref{eq:lb}, as well as a learner that samples from the optimal allocation, given by the lower bound, at all time steps. We also consider a learner that draws arms from the uniform distribution and in scenarios where the uniform distribution is not in the set of feasible exploration policies we project it onto the set and sample from the resulting distribution. 

In addition, we consider a na\"ive adaptation of Track-and-Stop~\citep{Kaufmann16}, which we call the \emph{Projected-Track-and-Stop} (PTnS). The \textit{PTnS algorithm computes the allocation as if it was solving the classical BAI problem and projects the allocation back to the feasible set when necessary}. Comparing CGE and CTnS with PTnS demonstrates (a) the importance of tracking the constrained lower bound to design an efficient algorithm, and, (b) the desired efficiency cannot be achieved just by tracking the unconstrained lower bound and projecting the corresponding allocation policy to the constrained set. Appendix \ref{app:add} contains additional experiments.\footnote{Code available at: \href{https://github.com/e-carlsson/constraint-pure-exploration}{https://github.com/e-carlsson/constraint-pure-exploration}}

\newpage

\begin{figure}[t!]
    \centering
    \begin{tabular}{cc}
     \subfloat[Characteristic time of the BAI problem as we vary $\mu_4$ and $\mu_5$.]{\includegraphics[width=0.23\textwidth]{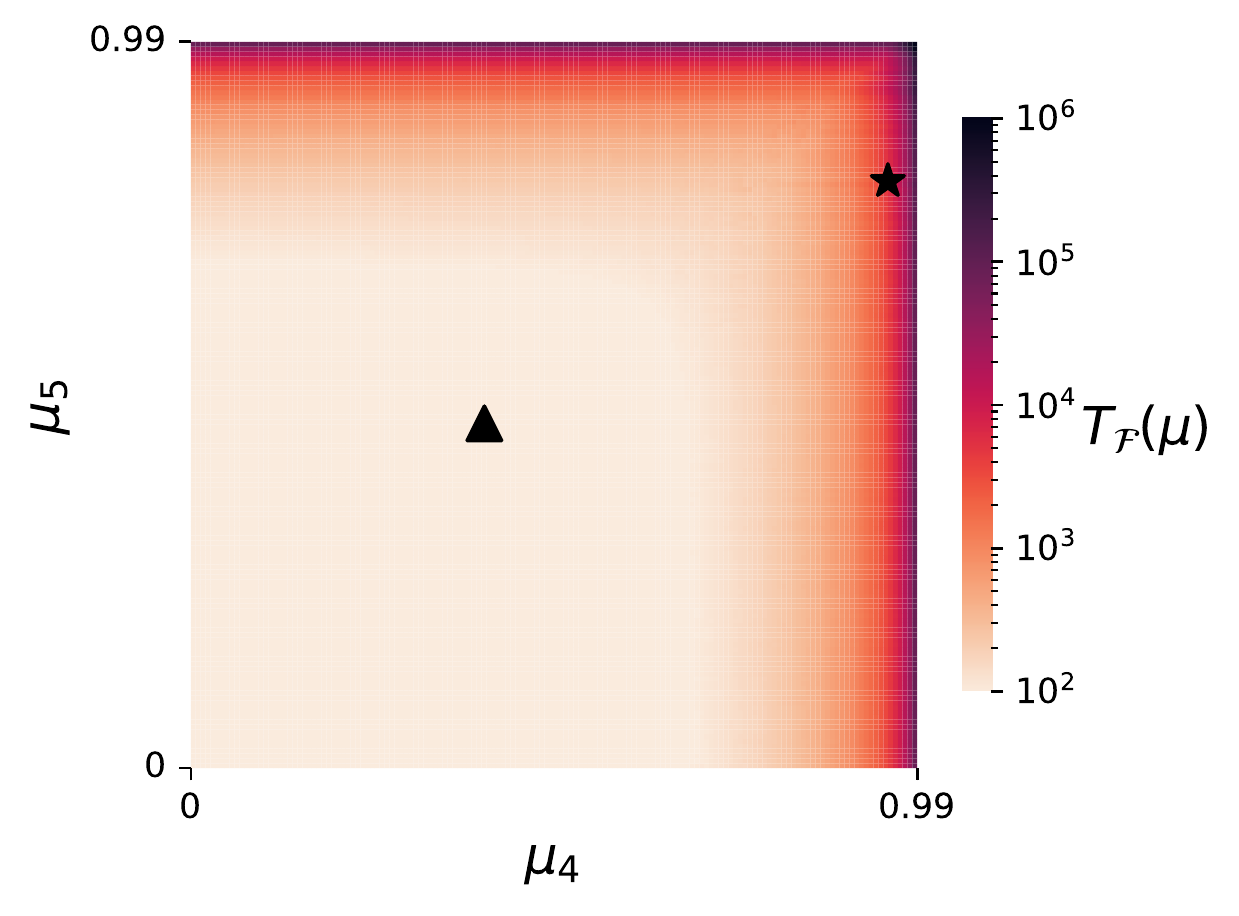}\label{fig:lb1}}   &
     \subfloat[Characteristic time of the constraint pure-exploration problem as we vary $\mu_4$ and $\mu_5$]{\includegraphics[width=0.23\textwidth]{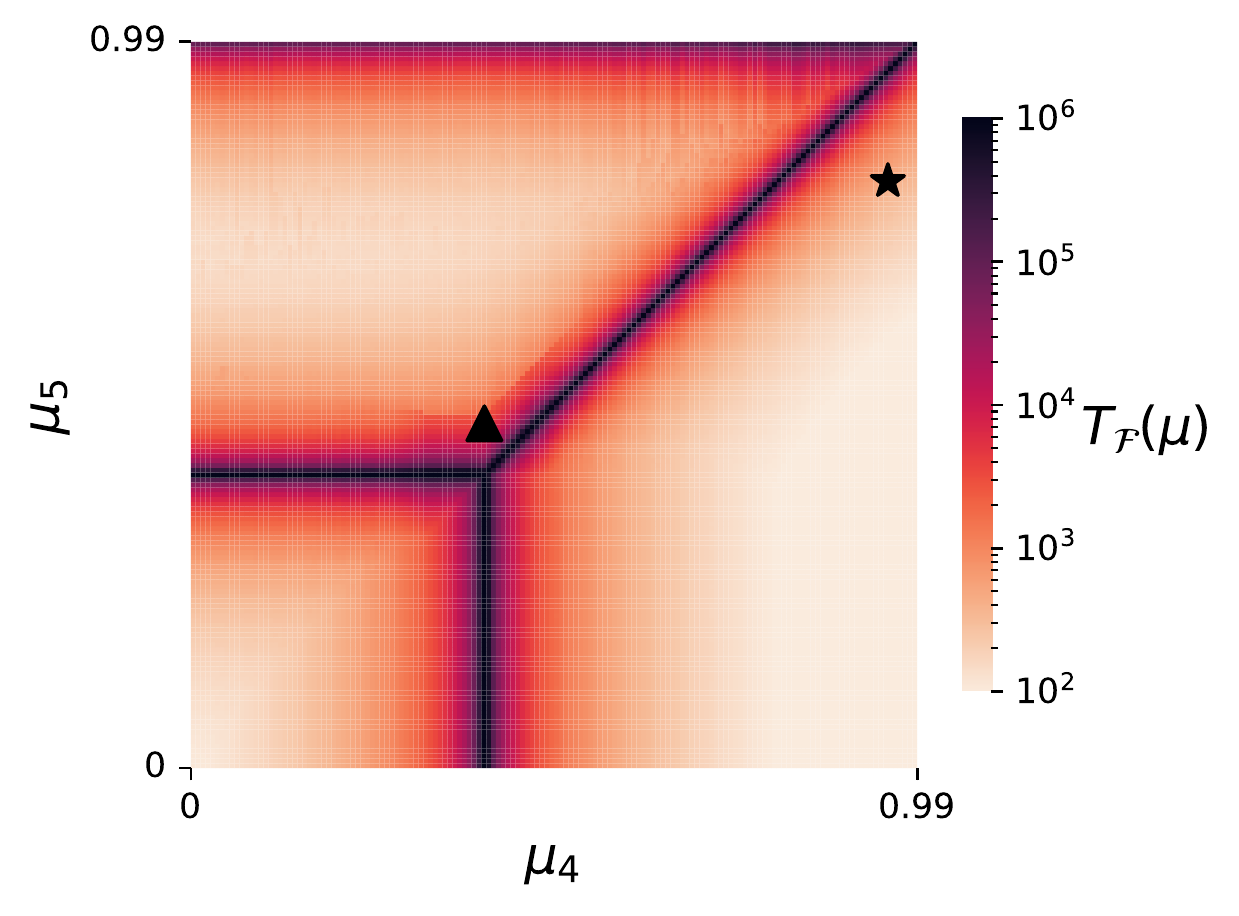}\label{fig:lb2}}
     \\
    \subfloat[Results (for $1000$ random seeds) on the instance highlighted as a triangle in Figure~\ref{fig:lb1} and Figure~\ref{fig:lb2}, $\bmu=(1, 0.5, 0.4, 0.4, 0.5)$, with constraints and $\delta=0.1$.]{\includegraphics[width=0.23\textwidth]{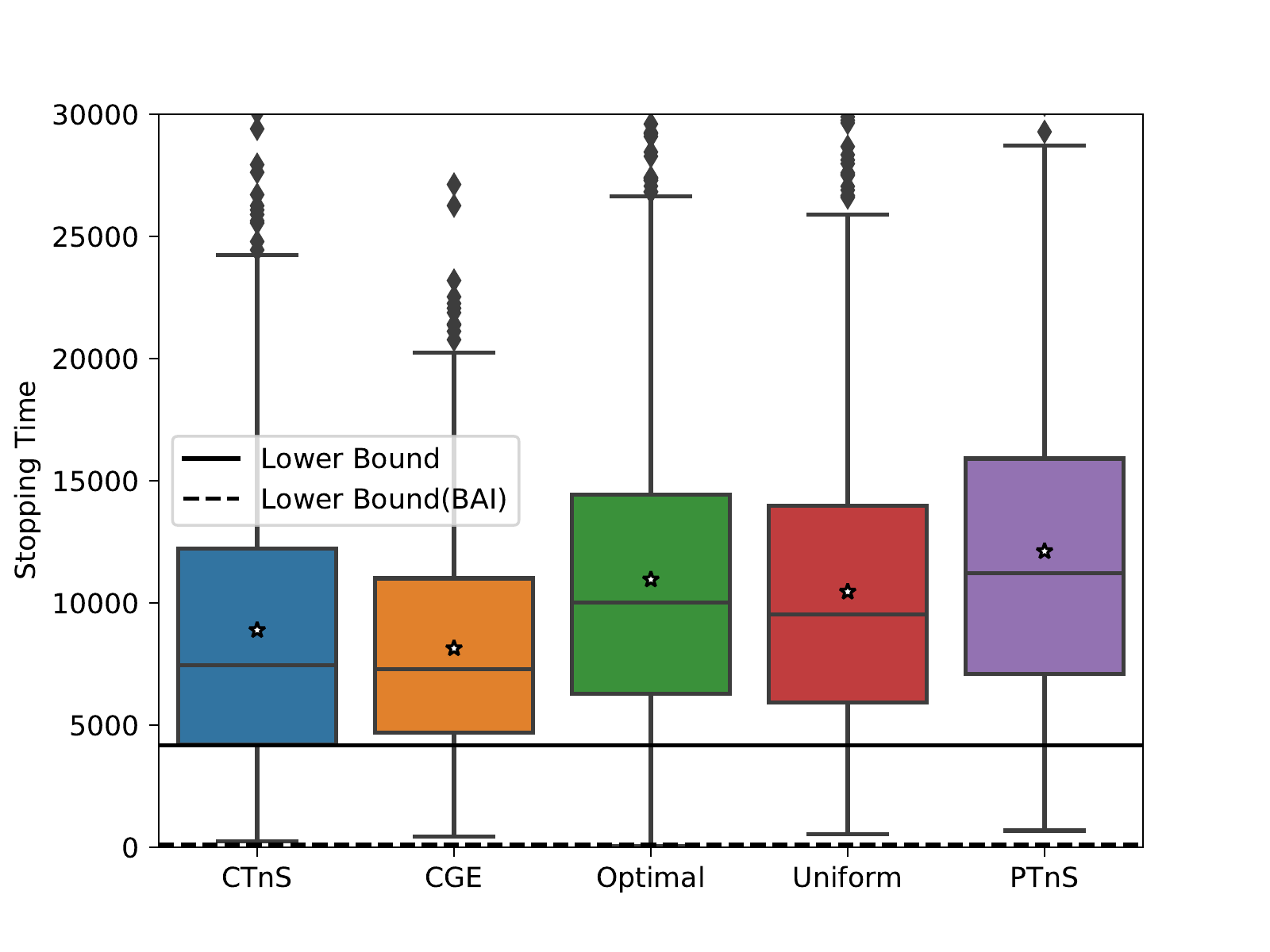}\label{fig:exp1}} &
     \subfloat[Results (for $1000$ random seeds) on the instance highlighted as a star in Figure~\ref{fig:lb1} and Figure~\ref{fig:lb2}, $\bmu=(1, 0.5, 0.4, 0.95, 0.8)$, with constraints and $\delta=0.1$.]{\includegraphics[width=0.23\textwidth]{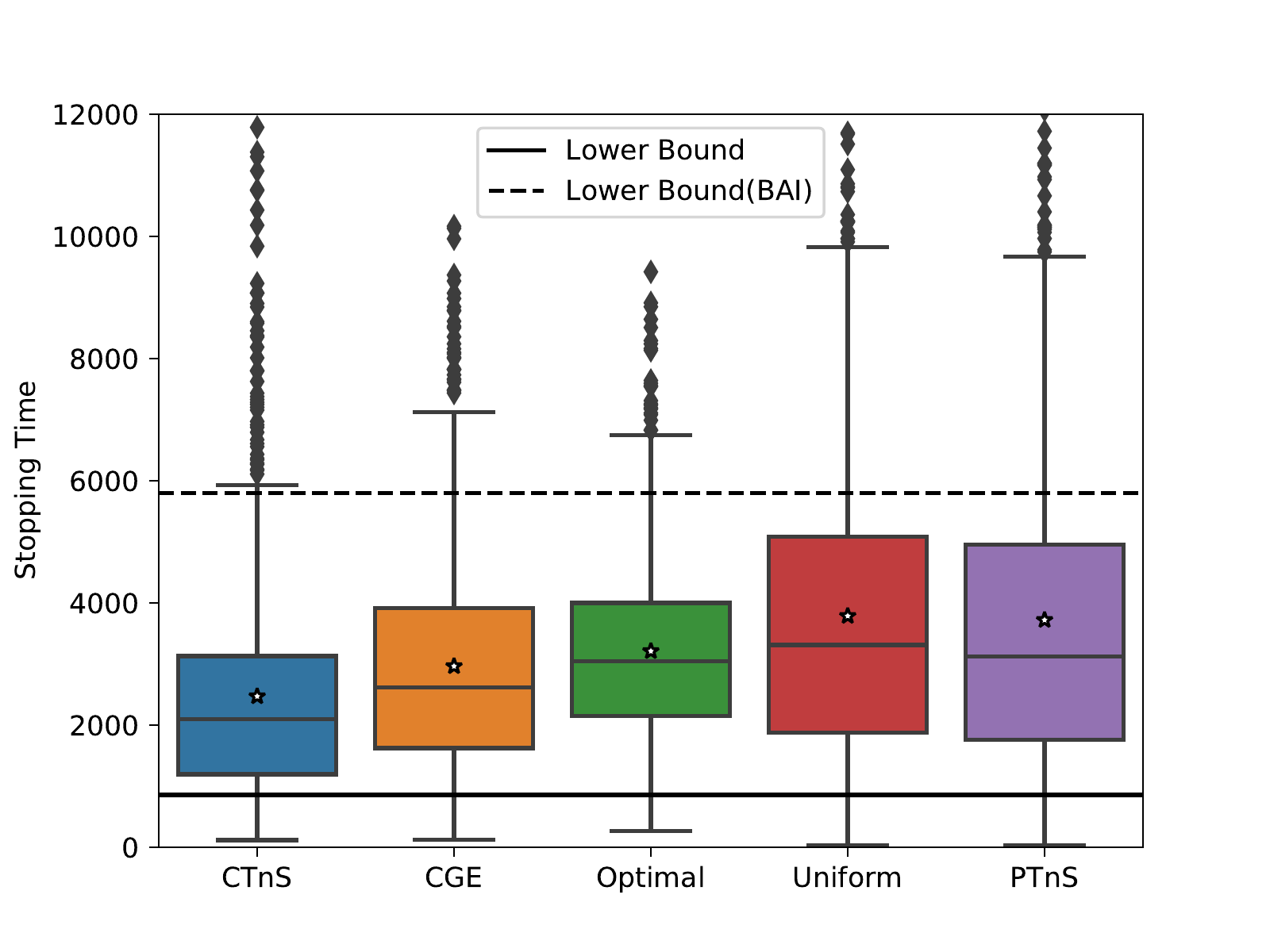}\label{fig:exp2}}
        \end{tabular}
    \caption{Figure~\ref{fig:lb1} and~\ref{fig:lb2} illustrate the hardness of the problem, i.e. the Characteristic time, changes in the $5$ arm instance $\bmu=(1.0, 0.5, 0.4, \mu_4, \mu_5)$ as we vary $\mu_4$ and $\mu_5$. Figure~\ref{fig:lb1} corresponds to the hardness in the BAI while Figure~\ref{fig:lb2} is the constraint setting with constraints $\pol_1 + \pol_2 \leq 0.5$ and $\pol_3 + \pol_4 \leq 0.5$. We clip the characteristic time at $10^6$ for visual purposes.}\label{fig:heatmap_exp}
\end{figure}

\textbf{Observation 1: Constraints alter the hardness of the problem.} In Figures ~\ref{fig:lb1} and \ref{fig:lb2} we illustrate how the hardness of a bandit instance $\bmu$ may differ once we introduce constraints, assuming anytime constraints. We consider the instance $\bmu=(1.0, 0.5, 0.4, \mu_4, \mu_5)$ and plot how the characteristic time $\chartime(\bmu)$ changes as we vary $\mu_4$ and $\mu_5$, Figure~\ref{fig:lb1} corresponds to the classical BAI, i.e. no constraints, and in Figure~\ref{fig:lb2} we have introduced the two constraints $\pol_1 + \pol_2 \leq 0.5$ and $\pol_3 + \pol_4 \leq 0.5$. We have highlighted two instances, one where the BAI problem is easy but the constraint problem is hard (black triangle) and one where the reverse is true (black star). We run the algorithms on these two instances in Figure~\ref{fig:exp1} and \ref{fig:exp2}, assuming anytime constraints, and observe that both algorithms operate close to the lower bound and outperforms the uniform allocation strategy. We also observe that the algorithms perform equally or better than the optimal learner, this is an interesting phenomena and have been observed earlier in other pure exploration scenarios~\citep{Degenne}. The PTnS does not account for the constraints, as well as CTnS and CGE, and has a sample complexity on par with uniform sampling. 

\begin{figure}[t!]
    \centering
    \begin{tabular}{cc}
    \subfloat[Anytime Constraints]{\includegraphics[width=0.23\textwidth]{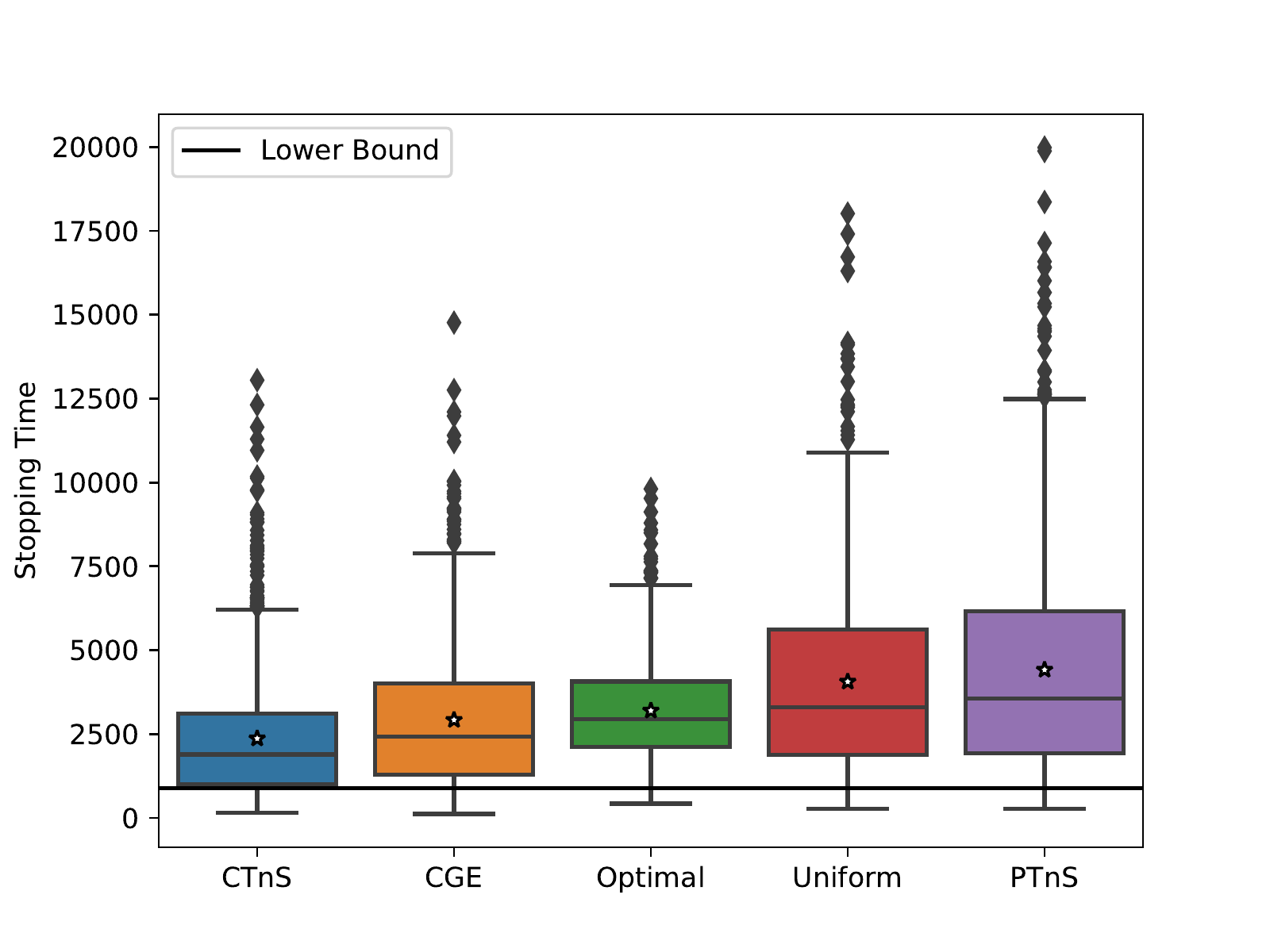}\label{fig:ptns_any}} &
    \subfloat[End-of-Time Constraints]{\includegraphics[width=0.23\textwidth]{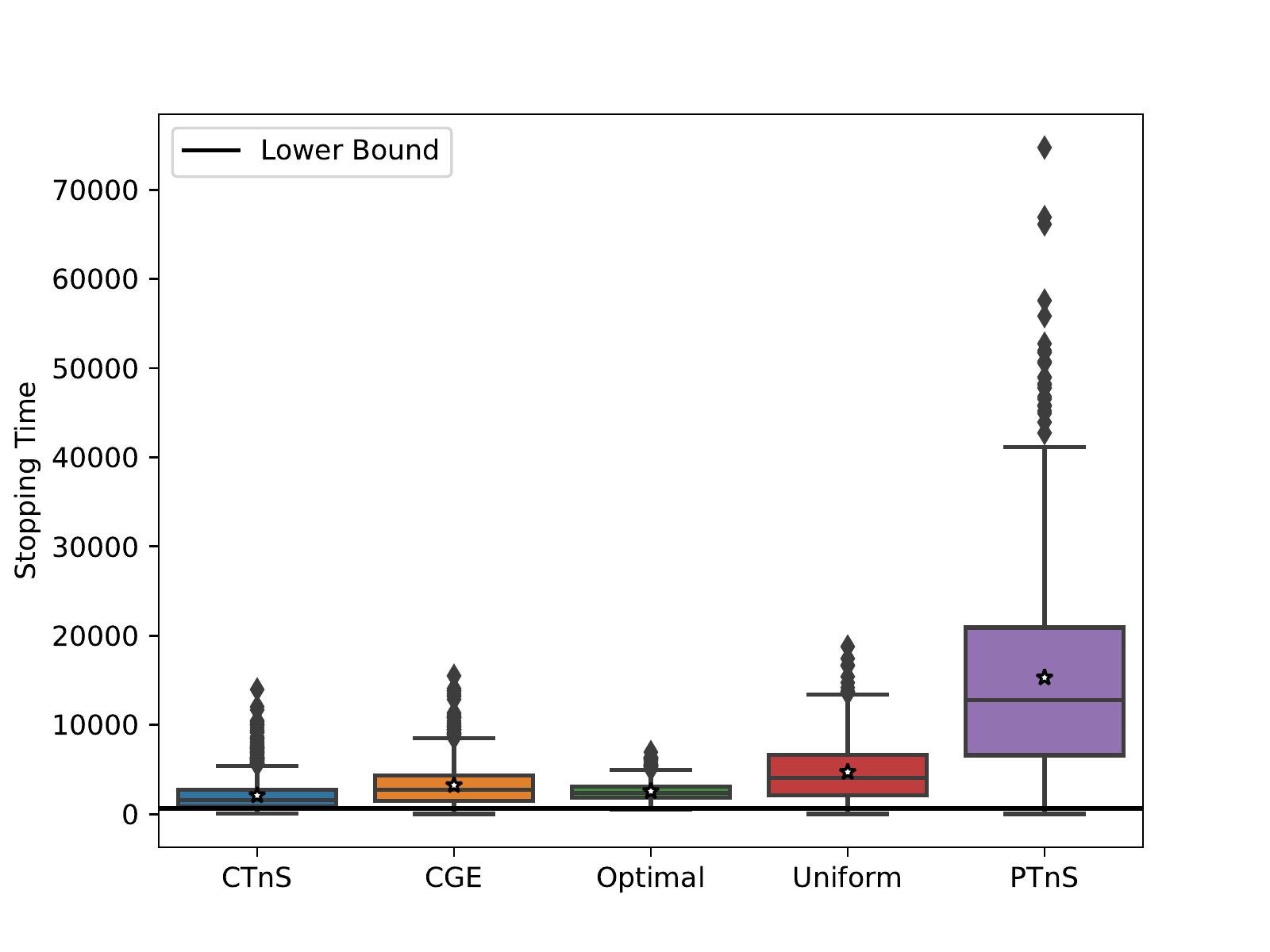}\label{fig:ptns_end}}
        \end{tabular}
    \caption{Problem instance with $8$ Gaussian arms with $\sigma=1$. The arm means are $\mu=[1.0, 0.7, 0.3, 0.0, -0.5, -1.0, -2.0, -3.0]$ and we have one constraint $7\pi_1 + 7\pi_2 + \pi_3 \leq 0.5$. The optimal policy is $\pi_3 = \pi_4=0.5$. Results for $\delta=0.1$ and $1000$ random seeds.}
    \label{fig:ptns}
\end{figure}

\textbf{Observation 2: Na\"ive projection cause high sample complexity.}
In Figure \ref{fig:ptns}, we consider an eight-armed bandit with Gaussian reward distributions.  We observe that PTnS performs the worst on this instance, specially in the end-of-time setting where it is outperformed by uniform sampling. This because in a BAI problem with the same $\bmu$ the hardness of the problem lies separating arm $1$ and $2$ but this doesn't have to be the case in the constraint bandit. The sub-optimality of PTnS in Figure~\ref{fig:ptns_any}, the anytime scenario, illustrates that na\"ively projecting the allocation onto the feasible set won't account for the constraints in a meaningful way. In Appendix~\ref{app:ptns} we further discuss these examples and compute the optimal allocations and the allocations PTnS converge to for each scenario. 
\begin{figure}[t!]
    \centering
    \begin{tabular}{cc}
     \subfloat[Anytime Constraints]{\includegraphics[width=0.23\textwidth]{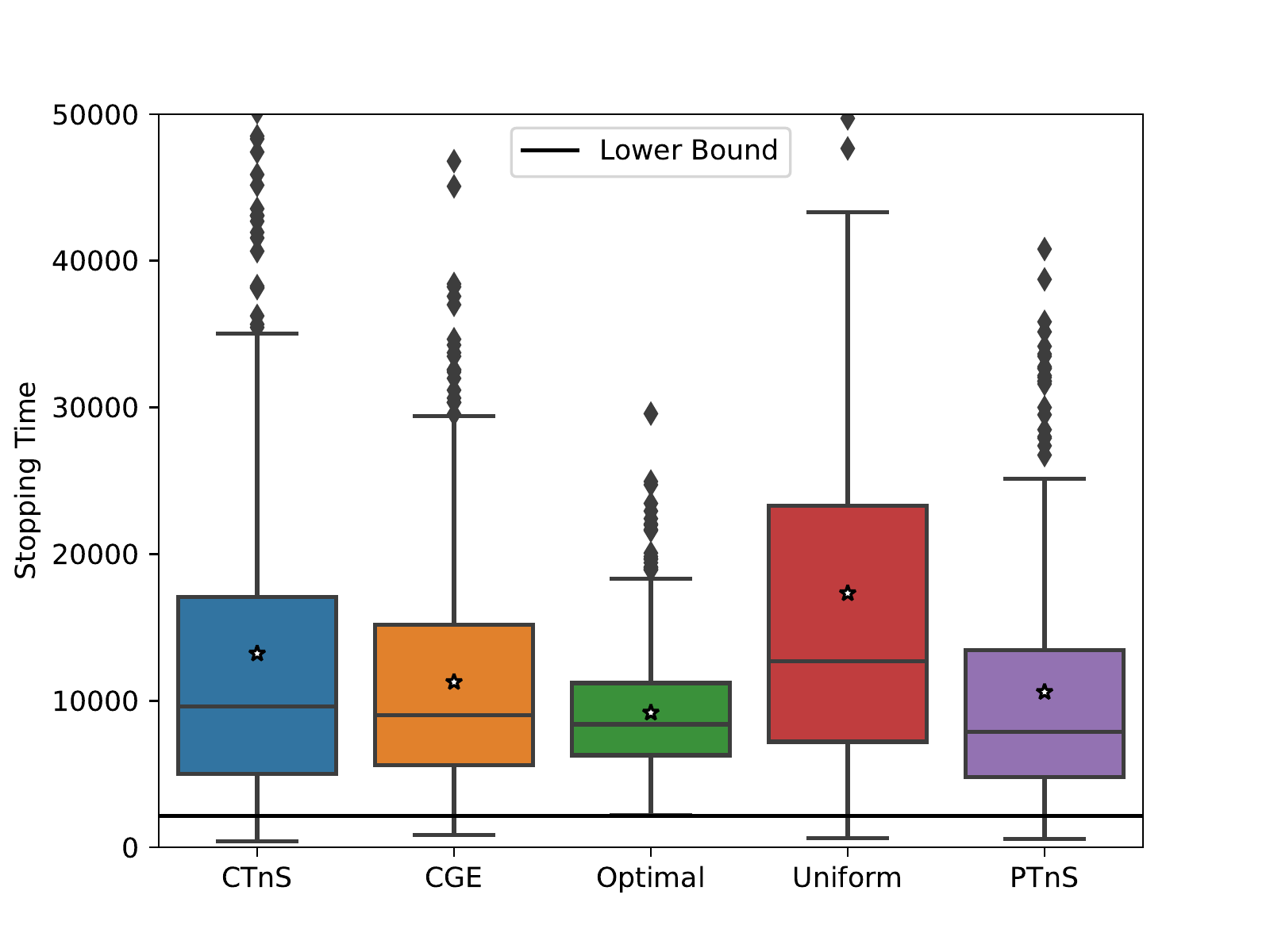}\label{fig:imdb_any} } &
     \subfloat[End-of-Time Constraints]{\includegraphics[width=0.23\textwidth]{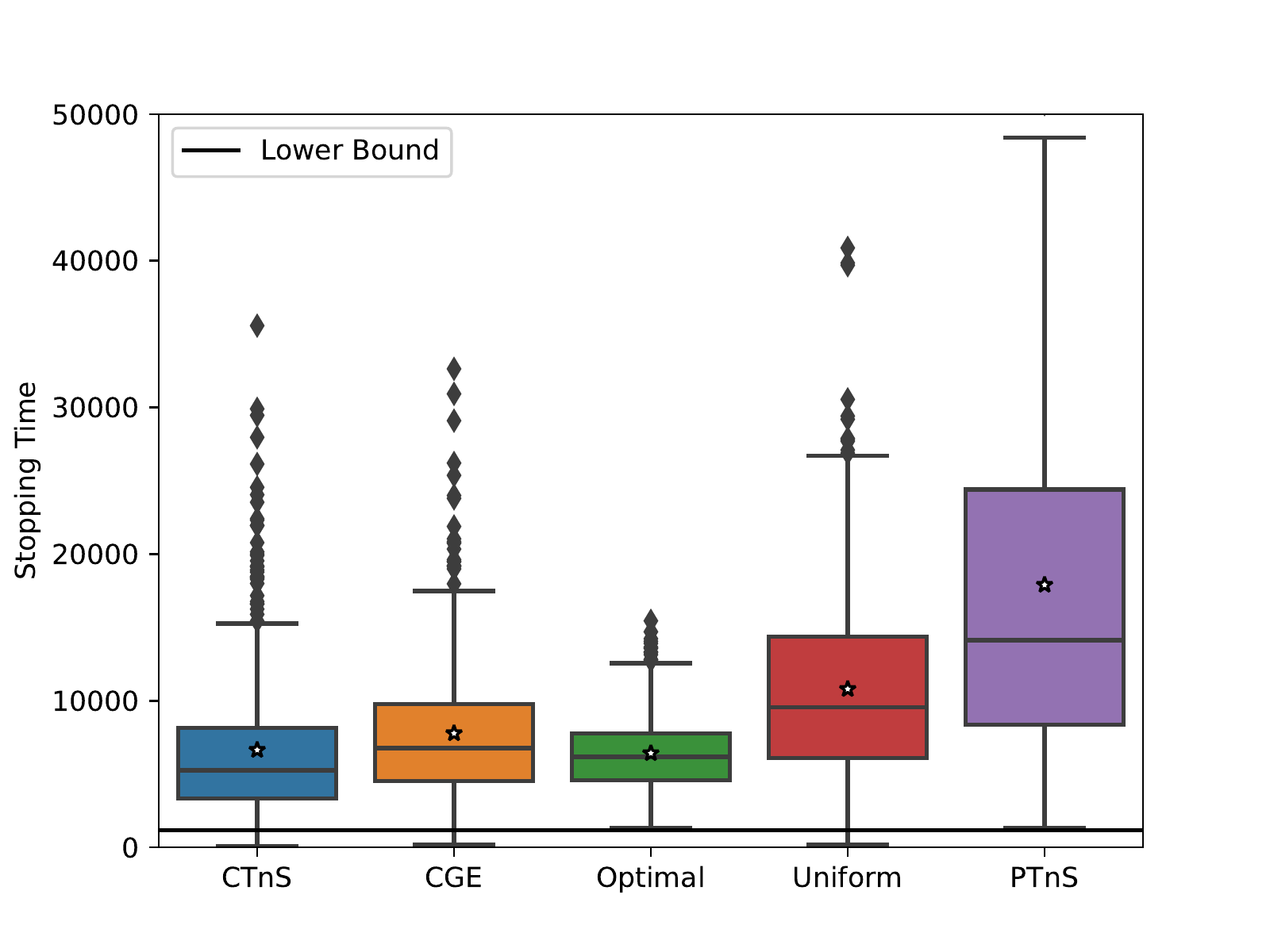}\label{fig:imdb_end}}
        \end{tabular}
    \caption{Experiments on IMDB dataset with $12$ movies and $\delta=0.1$. Each experiment was performed over $500$ random seeds.}\label{fig:exp_imdb} 
\end{figure}

\textbf{IMDB movie recommendation environment.} We construct a semi-synthetic task based on the widely used IMDB 50K Movie Dataset~\citep{imdb} which contains metadata on $k_0=50000$ movies including association with one or more of $d=23$ genres, as indicated by a binary matrix $X \in \{0,1\}^{m\times d}$. In our setting, actions correspond to recommending one out of a subset of $k \leq k_0$ movies. To create reward distributions for each movie, we simulate a population of $n_u=600$ users, each assigned $n_f=5$ favorite genres $f_i$ with weights $w_{if_i} = [20, 10, 5, 2, 2]$ and let $w_{ia}=0$ for $a\not\in f_i$. A score $s_{ia}$ for user $i$ and movie $a$ is created as follows, $
s_{ij} = \mbox{clip}(\lfloor \tilde{s}_{ia}/\sum_{a\in f_i}w_{ia}\cdot \sigma_0 + \sigma_1 \epsilon_{ia} \rceil \; ; 1, 5) \;\mbox{ where }\; \tilde{s}_{i\cdot} = w_i X^\top + w_0, $
$\epsilon_{ia} \sim U(0,1)$, $\sigma_0=5, \sigma_1=3$, and $\lfloor x \rceil$ indicates rounding of $x$ to the nearest integer. We construct the bandit environment by letting each movie $a$ be represented by an arm with reward $R_a \sim \mathcal{N}(\hat{\mu}_{s_{\cdot a}}, \hat{\sigma}^2_{s_{\cdot a}})$ determined by the mean and standard deviation of user reviews for the movie. We sample a subset of movies and search for the optimal policy that allocates at most $0.3$ to action movies, at least $0.3$ to drama movies and at least $0.3$ on family movies. Note that one movie might belong to more than one category. We present the result in Figure~\ref{fig:exp_imdb} for both the anytime scenario and the end of time scenario. We observe that CTnS and CGE outperform the uniform allocation strategy, which has a very high variance. We also observe a bigger difference between the algorithms under end of time constraints, this is reasonable since the set of plausible exploration policies is larger for that scenario. If the set of exploration policies is limited, there is little room for an algorithm to be adaptive. This is also captured in the fact that the lower bound for anytime constraints is always higher or equal to the bound for end-of-time constraints.

\section{Conclusions and future directions}

In this paper, we study the problem of pure exploration in bandits with linear constraints. We provide a generic lower bound for this setting that depends on an information-theoretic projection onto the boundary of the normal cone spanned by the active constraints at the optimal policy. We derive a closed-form lower bound for the case of Gaussian distributions and provide geometric insights into how constraints can make a problem easier or harder. Furthermore, we leverage the projection-based computation of the confusing instances to modify TnS~\citep{garivier2016optimal} and GE~\citep{Degenne} to corresponding CTnS and CGE versions for pure exploration in constraint bandits. We empirically evaluate the algorithms on synthetic and real data to assess the impact of constraints on the hardness of the problem.

One interesting future direction is learning when reward and constraints are unknown or partially unknown. Another future direction we deem very interesting is bandit with non-linear constraints as this would change this structure of the normal cone and the resulting projection.

\balance

\section*{Acknowledgements}
Emil Carlsson is funded by Chalmers AI Research Centre (CHAIR) and the Sweden-America foundation (SweAm). Fredrik D. Johansson is funded in part by the Wallenberg AI, Autonomous Systems and Software Program (WASP) funded by the Knut and Alice Wallenberg Foundation. Debabrota Basu acknowledges the Inria-Kyoto University Associate Team ``RELIANT'', the ANR young researcher (JCJC) award for the REPUBLIC project (ANR-22-CE23-0003-01), and the CHIST-ERA project CausalXRL (ANR-21-CHR4-0007).

The computations were enabled by resources provided by the National Academic Infrastructure for Supercomputing in Sweden (NAISS) partially funded by the Swedish Research Council through grant agreement no. 2022-06725.

\clearpage
\bibliography{lib}

\clearpage

\section*{Checklist}
 \begin{enumerate}

 \item For all models and algorithms presented, check if you include:
 \begin{enumerate}
   \item A clear description of the mathematical setting, assumptions, algorithm, and/or model. {\bf Yes}
   \item An analysis of the properties and complexity (time, space, sample size) of any algorithm. {\bf Yes}
   \item (Optional) Anonymized source code, with specification of all dependencies, including external libraries. {\bf Yes}
 \end{enumerate}

 \item For any theoretical claim, check if you include:
 \begin{enumerate}
   \item Statements of the full set of assumptions of all theoretical results. {\bf Yes}
   \item Complete proofs of all theoretical results.{\bf Yes}
   \item Clear explanations of any assumptions. {\bf Yes}  
 \end{enumerate}

 \item For all figures and tables that present empirical results, check if you include:
 \begin{enumerate}
   \item The code, data, and instructions needed to reproduce the main experimental results (either in the supplemental material or as a URL). {\bf Yes}
   \item All the training details (e.g., data splits, hyperparameters, how they were chosen). {\bf Yes}
         \item A clear definition of the specific measure or statistics and error bars (e.g., with respect to the random seed after running experiments multiple times). {\bf Yes}
         \item A description of the computing infrastructure used. (e.g., type of GPUs, internal cluster, or cloud provider). {\bf Yes, see appendix.}
 \end{enumerate}

 \item If you are using existing assets (e.g., code, data, models) or curating/releasing new assets, check if you include:
 \begin{enumerate}
   \item Citations of the creator If your work uses existing assets. {\bf Not Applicable}
   \item The license information of the assets, if applicable. {\bf Not Applicable}
   \item New assets either in the supplemental material or as a URL, if applicable. {\bf Not Applicable}
   \item Information about consent from data providers/curators. {\bf Not Applicable}
   \item Discussion of sensible content if applicable, e.g., personally identifiable information or offensive content. {\bf Not Applicable}
 \end{enumerate}

 \item If you used crowdsourcing or conducted research with human subjects, check if you include:
 \begin{enumerate}
   \item The full text of instructions given to participants and screenshots. {\bf Not Applicable}
   \item Descriptions of potential participant risks, with links to Institutional Review Board (IRB) approvals if applicable. {\bf Not Applicable}
   \item The estimated hourly wage paid to participants and the total amount spent on participant compensation. {\bf Not Applicable}
 \end{enumerate}

 \end{enumerate}

\appendix
\onecolumn
\part*{Appendix}

\section{Notations}
\renewcommand{\arraystretch}{1.5}
\begin{longtable}{p{2cm} p{.5cm} p{9.5cm}}
\caption{Notations}\label{tab:Notation}\\
\hline
 $K$ &$:$ & Number of arms.\\
 $\delta$ &$:$ & Confidence parameter.\\
$ \KL{x}{y}$ &$:$ & KL-divergence between two random variables with means $x$ and $y$.\\
$ \binaryKL{x}{y}$ &$:$ & KL-divergence between two Bernoulli random variables with means $x$ and $y$.\\
$\Domain$ &$\triangleq$ & $[\bmu_{\min}, \bmu_{\max}]^K$, i.e. the range of expected rewards \\
$\bmu$ &$:$ & True reward vector, $\bmu \in \Domain$.\\
$\hat{\bmu}_t$ &$:$ & Empirical means at time $t$ projected onto $\Domain$.\\
 $B$ &$:$ & Matrix defining the linear constraints, i.e. $B \pol \leq \boldsymbol{c}$.\\
 $\boldsymbol{c}$ &$:$ & Vector defining the upper bound in the linear constraints, $B \pol \leq \boldsymbol{c}$.\\
 $\Delta_{K-1}$ &$:$ & Simplex in $K$ dimensions.\\
 $\F$ &$\triangleq$ & $\{\pol \in \Delta_{K-1}: B\pol \leq \boldsymbol{c}\}$, i.e. the constrained policy space.\\
 $\pol$ &$:$ & A feasible policy over $K$ arms, i.e. $\pol \in \F$.\\
$\optpol$ or $\pol_{\bmu}^*$ &$:$ & Unique optimal policy for bandit instance $\bmu$, defined as $\pol_{\bmu}^* \triangleq \optpol \triangleq \argmax_{\pol \in \F} \bmu^\top \pol$. \\
$\neighbors$ &$:$ & Set of extreme points for $\pol'$, which share $K-1$ linearly independent constraints with $\optpol$. \\
$\Cone(\optpol)$ &$:$ & Normal cone spanned by the active constraints at $\optpol$. \\
$\Alt(\bmu)$ &$\triangleq$ & $\{\blambda \in \Domain: \max_{\pol \in \F} \blambda^\top \pol > \blambda^\top \pol_{\bmu}^*\}$, i.e. the set of alternative bandit instances. \\
 $\tau$ &$:$ & Random stopping time of a pure exploration algorithm.\\
 $\allocationset$ & $:$ & Set of possible exploration policies/allocations. \\
 $\chartime(\bmu)^{-1}$ & $\triangleq$ & $\sup_{\bw \in \allocationset} \inf_{\blambda \in \Alt(\bmu)} \sum_{a=1}^K w_a \KL{\mu_a}{\lambda_a}$, the characteristic time for the constrained policy space\\
 $D(\bw, \bmu, \F)$ & $:$ & Shorthand for $\inf_{\blambda \in \Alt(\bmu)} \sum_{a=1}^K w_a \KL{\mu_a}{\lambda_a}$. \\
 $D(\bw, \bmu, \blambda)$ & $:$ & Shorthand for $\sum_{a=1}^K w_a \KL{\mu_a}{\lambda_a}$. \\
 $w^*(\bmu)$ & $:$ & Set of optimal allocations for bandit instance $\bmu$. \\
 $H$ &$\triangleq$  &$\frac{2\sigma^2}{\|\boldsymbol{\Delta}\|_2^2}$ quantifies complexity of bandit instance $\bmu$\\
\bottomrule
\end{longtable}

\newpage
\section{Lower bound on sample complexity}\label{app:lb}

The following lemma by \citet{Kaufmann16} provides a general information-theoretic inequality that applies to any bandit model.
\begin{lemma}[\citet{Kaufmann16}]\label{lm:garivier}
    Let $\bmu$ and $\blambda$ be two bandit models with $K$ arms such that $\mu_a$ and $\lambda_a$ are mutually continuous. For any almost surely finite stopping time $\tau$ we have \begin{align}
        \sum_{a=1}^K \E_{\bmu}[N_{a, \tau}] \KL{\mu_a}{\lambda_a} \geq  \binaryKL{P_{\bmu}(\event)}{P_{\blambda}(\event)}
    \end{align}
    where $\event$ is any measurable event with respect to the filtration generated by the observed history. 
\end{lemma}

From Lemma~\ref{lm:garivier} we can directly derive a lower bound on the expected stopping time of any $\delta$-PAC algorithm in the constraint multi-armed bandit setting. We present this lower bound in Theorem~\ref{thm:lb} and the proof is virtually the same as the proof for the lower bound in \citet{garivier2016optimal}. We present it here for completeness.
\begin{theorem}[Lower bound on sample complexity under constraints]\label{thm:lb}
The stopping time $\tau$ of any $\delta$-PAC learner satisfy \begin{align}
     \E_{\bmu}[\tau] \geq \chartime(\bmu)\binaryKL{\delta}{1 - \delta}.
\end{align}

\end{theorem}
\begin{proof}
Let $\bmu$ and $\blambda \in \Alt(\bmu)$ be two bandit models with $K$ arms such that they do not share optimal policy, i.e. $\pol_{\bmu}^* \neq \pol_{\blambda}^*$. 

Let $\event$ denote the event of recommending $\pol_{\bmu}^*$ for any bandit instance at stopping using some $\delta$-PAC algorithm. Then using Lemma~\ref{lm:garivier}, and $\delta$-correctness of $\pol_{\bmu}^*$ for $\bmu$, we have 
\begin{align*}
    \sum_{a=1}^K \E_{\bmu}[N_{a, \tau}] \KL{\mu_a}{\lambda_a} \geq  \binaryKL{1- \delta}{\delta} = \binaryKL{\delta}{1 - \delta}.
\end{align*}
Further, we multiple and divide by $\E_{\bmu}[\tau]$ which yields 
\begin{align*}
    \sum_{a=1}^K \E_{\bmu}[N_{a, \tau}] \KL{\mu_a}{\lambda_a} &= \E_{\bmu}[\tau]\sum_{a=1}^K \frac{\E_{\bmu}[N_{a, \tau}]}{\E_{\bmu}[\tau]} \KL{\mu_a}{\lambda_a} \\
    &= \E_{\bmu}[\tau]\sum_{a=1}^K w_a \KL{\mu_a}{\lambda_a} \geq  \binaryKL{\delta}{1 - \delta}\: ,
\end{align*}
where $w_a \triangleq \frac{\E_{\bmu}[N_{a, \tau}]}{\E_{\bmu}[\tau]}$, and $\sum_{a=1}^K w_a =1$.

Since the above inequality is true for any $\blambda \in \Alt(\bmu)$, we have 
\begin{align*}
   \inf_{\blambda \in \Alt(\bmu)}  \E_{\bmu}[\tau]\sum_{a=1}^K w_a \KL{\mu_a}{\lambda_a} = \E_{\bmu}[\tau] \inf_{\blambda \in \Alt(\bmu)}  \sum_{a=1}^K w_a \KL{\mu_a}{\lambda_a} \geq  \binaryKL{\delta}{1 - \delta} \: .
\end{align*}
The equality is due to the fact that $\E_{\bmu}[\tau]$ is independent of $\blambda$.

Now, we further maximise over $w_a$ to get 
\begin{align*}
    \E_{\bmu}[\tau] \sup_{\bw \in \allocationset} \inf_{\blambda \in \Alt(\bmu)} \sum_{a=1}^K w_a \KL{\mu_a}{\lambda_a} \geq  \binaryKL{\delta}{1 - \delta}.
\end{align*}
Finally, using the definition of the characteristic time $\chartime(\bmu)$ yields 
\begin{align*}
    \E_{\bmu}[\tau] \geq \chartime(\bmu)\binaryKL{\delta}{1 - \delta} \: .
\end{align*}
\end{proof}

\subsection{Projection lemma for ${D(\bw, \bmu, \F)}$: Proof of Lemma~\ref{lm:proj}}
To derive the key properties of the optimal solution and the set of optimal allocations, as presented in Lemma~\ref{lm:properties}, we first explicate the set of optimal solutions, and then, use  Berge's theorem (Theorem~\ref{thm:berge}). 

\textbf{Step 1:} Recall that 
\begin{align*}
    \Alt(\bmu) = \left\{\blambda \in \Domain: \blambda \notin \Cone(\optpol) \right\} \: ,
\end{align*}
where the normal cone is expressed as 
\begin{align*}
    \Cone(\optpol) = \bigcap_{\pol' \in \neighbors} \left \{ \blambda \in \Domain: \blambda^\top (\optpol - \pol') \geq 0 \right \} \: .
\end{align*}
This is due to the fact that if $\optpol$ is not the optimal policy under the environment $\blambda$, there exists an improving direction in the simplex algorithm, i.e. a neighbor $\pol'$, such that $\blambda^\top (\optpol - \pol') < 0$. 

Now, since the set of alternative hypotheses is the compliment of the normal cone, we write 
\begin{align}\label{eq:decomposed_cone}
    \Alt(\bmu) = \bigcup_{\pol' \in \neighbors} \left \{ \blambda: \blambda^\top (\optpol - \pol') < 0 \right \}.
\end{align}
Applying Equation~\eqref{eq:decomposed_cone} in $D(\bw, \bmu, \F)$ leads to,
\begin{align*}
    D(\bw, \bmu, \F) &= \inf_{\blambda \in \Alt(\bmu)} \sum_{a=1}^K w_a \KL{\mu_a}{\lambda_a}
    = \min_{\pol' \in \neighbors} \inf_{\blambda: \blambda^\top(\pol' - \optpol) <0} \sum_{a=1}^K w_a \KL{\mu_a}{\lambda_a} \: . 
\end{align*}

\textbf{Step 2:} What remains to be shown is that the $\inf$ is attained by some $\lambda$ on $\blambda^\top(\pol' - \optpol) = 0$. 

For some $\pol'  \in \neighbors$ take an arbitrary $\blambda' \in \left\{\blambda: \blambda^\top(\pol' - \optpol) <0\right\}$. There exists an $\blambda'' \in \left \{ \blambda \in \Domain: \blambda^\top (\optpol - \pol') = 0 \right \}$ such that $|\mu_a - \lambda_a'| \geq |\mu_a - \lambda_a''|$ $\forall a$ due to the convexity of $\Domain$.  The mapping $y \rightarrow \KL{x}{y}$ is an increasing function on the domain $y>x$ and a decreasing function on $y < x$ which implies that 
\begin{align}\label{eq:close_boundary}
    \sum_{a=1}^K w_a \KL{\mu_a}{\lambda_a'} \geq \sum_{a=1}^K w_a \KL{\mu_a}{\lambda_a''}.
\end{align}
There exists a sequence $\{\blambda_t\}_{t=1}^\infty \subset \left\{ \blambda: \blambda^\top (\optpol - \pol') < 0 \right \} $ such that $\blambda_0 = \blambda'$ and $\lim_{t\rightarrow \infty} \blambda_t= \blambda''$. Hence, we can for any $\blambda'$ get arbitrary close to some $\blambda''$ such that Equation~\eqref{eq:close_boundary} holds. 

Due to continuity of $\KL{x}{.}$, the $\inf$ is attained by some $\blambda'' \in \left \{ \blambda \in \Domain: \blambda^\top (\optpol - \pol') = 0 \right \}$. Hence, we conclude the proof.

\subsection{Properties of $D(\bw, \bmu, \F)$: Proof of Theorem~\ref{lm:properties}}
\underline{Property (a-b).}  We first note that the function $D(\bw, \bmu, \blambda) \triangleq \sum_{a=1}^K w_a \KL{\mu_a}{\lambda_a}$ is continuous in all elements. Take any $(\bw, \bmu)$ such that the optimal policy in $\F$ is unique.  Let $(\bw_t, \bmu_t)_{t\geq 1}$ be a sequence in $\allocationset \times \Domain$ such that 
\begin{align*}
      (\bw_t, \bmu_t) \xrightarrow{t \rightarrow \infty} (\bw, \bmu).
  \end{align*}
  Further,  for any $\epsilon > 0$ there exists a $t'\geq 1$ such that $||(\bw, \bmu) - (\bw_t, \bmu_t) ||_2 < \epsilon $ and $\Alt(\mu) = \Alt(\mu_t)$ $\forall t \geq t'$. By continuity of $D(\bw, \bmu, \blambda)$ we have that for any $\epsilon' > 0$ there exists exists an $t'' \geq 1$ such that for $t\geq t''$, we have 
  \begin{align*}
      |D(\bw_t, \bmu_t, \blambda) - D(\bw, \bmu, \blambda)| \leq \epsilon', \forall \lambda \in \R^K.
  \end{align*}
  Thus, by taking $t \geq t', t''$ leads to 
  \begin{align*}
      |D(\bw, \bmu, \F) - D(\bw_t, \bmu_t, \F)| &= \left|\inf_{\lambda \in \Alt(\mu)}D(\bw, \bmu, \lambda) - \inf_{\lambda \in \Alt(\mu_t)}D(\bw_t, \bmu_t, \lambda) \right| \\
  & \leq \left|\inf_{\lambda \in \Alt(\mu)}\left(D(\bw, \bmu, \lambda) - D(\bw_t, \bmu_t, \lambda)\right) \right| \\
  & \leq \epsilon' \: ,
  \end{align*}
  which establishes the continuity properties.
  
\underline{Property (c).}  The upper hemicontinuity of $w^*(\bmu)$ and continuity of $D(\bmu, \F)$ follows from Berge's maximum theorem, see Theorem~\ref{thm:berge}, by letting $f(x, \theta) = D(\bw, \bmu, \F)$ and $C(\theta)=\allocationset$. As a consequence of Berge's theorem (Theorem~\ref{thm:berge}), we substitute the $\sup_w$ with $\max_w$.

\underline{Property (d).}  The convexity of the set $w^*(\bmu)$ follows from the fact that it is the set of optimal solutions to $\max_{w \in \allocationset} D(\bw, \bmu, \F)$ and $D(\bw, \bmu, \F)$ is concave (Specifically, it is linear in $\bw$).
  
\subsection{Projective representation of characteristic time for Gaussians: Proof of Theorem~\ref{cor:projection}}
For two bandit instances $\bmu$ and $\blambda$ consisting of Gaussian distributions with same variance $\sigma^2$, we have 
\begin{align*}
   D(\bw, \bmu, \F) &= \min_{\blambda: \blambda^\top (\optpol - \pol') = 0} \sum_{a=1}^K w_a\frac{1}{2\sigma^2}(\mu_a - \lambda_a)^2 \: .
\end{align*}
Now, by introducing the Lagrange multiplier $\gamma$, we obtain 
\begin{align}\label{eq:lagrangian}
   L(\gamma, \blambda) \triangleq \frac{1}{2\sigma^2}\sum_{a=1}^K w_a(\mu_a - \lambda_a)^2 - \gamma \blambda^\top (\optpol - \pol').
\end{align}
For brevity, we denote $v \triangleq (\optpol - \pol')$.

Computing the gradient $\nabla_{\lambda} L(\gamma, \blambda)$ and equating it to $0$ yields \begin{align*}
    \lambda_a &= \mu_a + \frac{\gamma \sigma^2}{w_a} v_a.
\end{align*}
Substituting $\lambda_a$ in Equation~\eqref{eq:lagrangian} yields
\begin{align}
    L(\gamma) = \min_\lambda L(\gamma, \blambda) &= \frac{\sigma^2 \gamma^2}{2} \sum_{a=1}^K \frac{v_a^2}{w_a} -
    \gamma \bmu^\top v - \sum_{a=1}^K \frac{\gamma^2 \sigma^2}{w_a} v_a^2 \notag \\
    &= -\frac{\sigma^2 \gamma^2}{2} \sum_{a=1}^K \frac{v_a^2}{w_a} -
    \gamma \bmu^\top v \: . \label{eq:L_gamma}
\end{align}
Maximizing over $\gamma$ yields 
\begin{align*}
    \gamma = \frac{- \bmu^\top v}{\sigma^2 \sum_a \frac{v_a^2}{w_a}} \: ,
\end{align*}
and putting it back in Equation~\eqref{eq:L_gamma} gives the final expression of $\lambda_a$
\begin{align}
    \lambda_a = \mu_a - \frac{v_a}{w_a} \left( \frac{\bmu^\top v }{\sum_a \frac{v_a^2}{w_a}}\right)  \: .
\end{align}

\newpage
\subsection{Upper and lower bounding the characteristic time: Proof of Corollary~\ref{prop:lb_lower_bound}}
\textbf{Lower bound on the characteristic time:} To lower bound $\chartime(\bmu) $, we need to upper bound the RHS in Equation~\eqref{eq:char_time}, i.e. $\chartime(\bmu)^{-1} = \sup_{\bw} \min_{\lambda} \sum_{a} w_a \KL{\mu_a}{\lambda_a}$. 

\textbf{Step 1:} We first observe that 

$$\sup_{\bw} \min_{\lambda} \sum_{a} w_a \KL{\mu_a}{\lambda_a} = \max_{\bw} \min_{\lambda} \sum_{a} w_a \KL{\mu_a}{\lambda_a},$$ 

due to Berge's theorem. 
Further, the max-min inequality gives  

$$\max_{\bw} \min_{\lambda} \sum_{a} w_a \KL{\mu_a}{\lambda_a} \leq  \min_{\lambda} \max_{\bw} \sum_{a} w_a \KL{\mu_a}{\lambda_a}.$$ 

\textbf{Step 2:} We proceed to upper bound $\max_{\bw} \sum_{a} w_a \KL{\mu_a}{\lambda_a}$ for each neighbor $\pol' \in \neighbors$ independently. 

For a fixed $\pol' \in \mathcal{V}_{\F}(\optpol)$, Theorem~\ref{cor:projection} tells us that \begin{align*}
       \min_{\lambda: \lambda^\top(\optpol - \pol')=0} \sum_{a=1}^K w_a \KL{\mu_a}{\lambda_a} &= \frac{\gamma^2 }{2\sigma^2} \sum_{a=1}^K \frac{(\optpol - \pol')^2_a}{w_a}  \\
     &= \left(\frac{\bmu^\top \left(\optpol - \pol' \right) }{\sum_a \frac{(\optpol - \pol')^2}{w_a}} \right)^2 \frac{1}{2\sigma^2} \sum_{a=1}^K \frac{(\optpol - \pol')^2_a}{w_a} \\
     &= \frac{1}{2\sigma^2} \frac{\left(\bmu^\top \left(\optpol - \pol' \right) \right)^2}{\sum_{a=1}^K \frac{(\optpol - \pol')^2_a}{w_a}}.
\end{align*}

\textbf{Step 3:} We further minimize the expression $\frac{(\optpol - \pol')^2_a}{w_a}$ under the constraint $\sum_a w_a = 1$. 

Using Langrange multiplier technique, we get 
\begin{align*}
     w_a = \frac{|(\optpol - \pol')|_a}{\sum_{a=1}^K|(\optpol - \pol')|_a}
\end{align*}
which yields that $\frac{(\optpol - \pol')^2_a}{w_a} \geq \|\optpol - \pol'\|_1^2$. Hence, 
\begin{align*}
    \frac{1}{2\sigma^2} \frac{\left(\bmu^\top \left(\optpol - \pol' \right) \right)^2}{\sum_{a=1}^K \frac{(\optpol - \pol')^2_a}{w_a}} \leq \frac{1}{2\sigma^2} \frac{\left(\bmu^\top \left(\optpol - \pol' \right) \right)^2}{\|\optpol - \pol'\|_1^2} \leq \frac{1}{2\sigma^2} \frac{\left(\bmu^\top \left(\optpol - \pol' \right) \right)^2}{\|\optpol - \pol'\|_2^2}.
\end{align*}
Here, the last part is exactly , $\frac{1}{2}d^2_{\pol'}$, i.e. the squared distance between $\bmu$ and the hyperplane $\optpol - \pol=0$.

Thus, we conclude the lower bound.

\textbf{Upper bound on the characteristic time:} To obtain the upper bound, we aim to lower bound the inverse $\chartime(\bmu)^{-1} = \sup_{\bw} \min_{\lambda} \sum_{a} w_a \KL{\mu_a}{\lambda_a} $. 

We let ${w}_a = \frac{1}{K}, \forall a$, and observe that

$$\max_{w} \min_{\lambda} \sum_a w_a \KL{\mu_a}{\lambda_a} \geq \min_{\lambda} \frac{1}{K}\sum_a \KL{\mu_a}{\lambda_a}.$$ 
For some $\pol' \in \mathcal{V}(\optpol)$ and using Theorem~\ref{cor:projection} with $w_a=\frac{1}{K}, \forall a$, we get 
\begin{align*}
    \frac{1}{K}\sum_a \KL{\mu_a}{\lambda_a} &= \frac{1}{2\sigma^2 K } \frac{\left(\bmu^\top \left(\optpol - \pol' \right) \right)^2}{\|\optpol - \pol' \|_2^2} = d_{\pol'}^2 \frac{1}{2\sigma^2K}
\end{align*}
This concludes the upper bound on the characteristic time.

\subsection{Impact of linear constraints: Proof of Corollary~\ref{prop:spectral}}
\textbf{Step 1: Neighboring policies and rank-1 update.} let $\hat{B} \in \R^{K \times K}$ be a set of linearly independent constraints at $\optpol$ and $\hat{\boldsymbol{c}}$ be the corresponding values in $\boldsymbol{c}$ such that $\pol^* =\hat{B}^{-1} \hat{\boldsymbol{c}}$. For any $\pol' \in \neighbors$ we let $ B'^{-1}$ and $ \boldsymbol{c}'$ be the constraints such that $\pol' = B'^{-1} \boldsymbol{c}'$.

Specifically, $B'$ and $\boldsymbol{c}'$ can be retrieved from the following rank-1 updates 
\begin{align*}
    B' &= \hat{B} + \boldsymbol{e}_r(\boldsymbol{b}_r' - \hat{\boldsymbol{b}}_r)^\top \: ,\\
    \boldsymbol{c}' &= \widehat{\boldsymbol{c}} + (c_r' - c_r)  \boldsymbol{e}_r \: ,
\end{align*}
where $\hat{\boldsymbol{b}}_r$ a column vector corresponding to the constraint on the $r$-th row of $\hat{B}$ that we swap with $\boldsymbol{b}_r'$ in order to get $B'$ and $\boldsymbol{e}_r$ a column vector with all elements equal to $0$ except the $r$-th element which is equal to $0$.
Similarly, $(c_r' - c_r) \neq 0$ is the change that we perform on the $r$-th element in $\hat{\boldsymbol{c}}$ to get $\boldsymbol{c}'$. 

\textbf{Step 2: From perturbation in constraints to perturbations in policies.} Now, we observe that 
\begin{align*}
    B'\pol' - \hat{B}\optpol = (c_r' - c_r)\boldsymbol{e_r}\: .
\end{align*}
Since $\hat{B}$ is invertible, further rearrangement yields 
\begin{align*}
    \pol' - \optpol &= \hat{B}^{-1}\left((c_r' - c_r)  \boldsymbol{e}_r + \boldsymbol{e}_r(\hat{\boldsymbol{b}}_r - \boldsymbol{b}_r')^\top \pol' \right)\\
    &=\hat{B}^{-1}\left((c_r' - c_r) \boldsymbol{e}_r + \boldsymbol{e}_r\hat{\boldsymbol{b}}_r^\top \pol' - \boldsymbol{e}_r\boldsymbol{b}_r'^\top \pol' \right)\\
    &=\hat{B}^{-1}\left((c_r' - c_r) \boldsymbol{e}_r + \boldsymbol{e}_r\hat{\boldsymbol{b}}_r^\top \pol' - c_r'\boldsymbol{e}_r \right)\\
    &=\hat{B}^{-1}\left((\hat{\boldsymbol{b}}^\top_r\pol' - c_r) \boldsymbol{e}_r \right)
\end{align*}
The last part is the slack of $\pol'$ at the $r$-th constraint in $\hat{B}$, hereby referred to as $\xi$.

We bound the norm of $\hat{B}^{-1}\boldsymbol{e}_r$ as follows \begin{align*}
    \sigma_{\min}(\hat{B}^{-1}) = \inf_{\boldsymbol{v}: \|v\|_2=1}\|\hat{B}^{-1} v \|_2 \leq \|\hat{B}^{-1} \boldsymbol{e}_r \|_2 \leq \sup_{\boldsymbol{v}: \|v\|_2=1}\|\hat{B}^{-1} v \|_2 = \sigma_{\max}(\hat{B}^{-1})
\end{align*}
where $\sigma_{\min}(\hat{B}^{-1}) $ and $\sigma_{\max}(\hat{B}^{-1}) $ denote the smallest and largest singular value of $\hat{B}^{-1}$.
From the properties of the inverse, we get 
\begin{align*}
  \frac{1}{\sigma_{\max(\hat{B})}} \leq \|\hat{B}^{-1} \boldsymbol{e}_r \|_2 \leq   \frac{1}{\sigma_{\min(\hat{B})}}.
\end{align*}

Thus, we obtain a lower and upper bound on the perturbation in policies
\begin{align}
   \frac{|\xi|}{\sigma_{\max(\hat{B})}} \leq \|\pol' - \optpol\|_2 \leq   \frac{|\xi|}{\sigma_{\min(\hat{B})}}.
\end{align}

Now, using this new representation of change in policy in terms of the slacks in the constraints, we derive our two results.

\textbf{Step 3 for Part (a): A perspective of the zero-sum game.}
To get the expression in Equation~\eqref{eq:slack_diff} we simply take the expression for $\optpol - \pol'$, developed in the previous step, and plug into the expression of the characteristic time in Theorem~\ref{cor:projection}.  Hence, \begin{align*}
    \frac{1}{2\sigma^2} \frac{ \|\optpol - \pol'\|^2_{\bmu\bmu^{\top}}}{ \|\optpol - \pol'\|^2_{\mathrm{Diag}(1/w_a)}} &=  \frac{1}{2\sigma^2} \frac{ \|\hat{B}^{-1}\left((\hat{\boldsymbol{b}}^\top_r\pol' - c_r) \boldsymbol{e}_r \right)\|^2_{\bmu\bmu^{\top}}}{ \|\hat{B}^{-1}\left((\hat{\boldsymbol{b}}^\top_r\pol' - c_r) \boldsymbol{e}_r \right)\|^2_{\mathrm{Diag}(1/w_a)}} \\
    &=\frac{1}{2\sigma^2} \frac{ \|\hat{B}^{-1}\left(\xi \boldsymbol{e}_r \right)\|^2_{\bmu\bmu^{\top}}}{ \|\hat{B}^{-1}\left(\xi \boldsymbol{e}_r \right)\|^2_{\mathrm{Diag}(1/w_a)}} \\
    &= \frac{1}{2\sigma^2} \frac{ \|\hat{B}^{-1}\left(\boldsymbol{e}_r \right)\|^2_{\bmu\bmu^{\top}}}{ \|\hat{B}^{-1}\left(\boldsymbol{e}_r \right)\|^2_{\mathrm{Diag}(1/w_a)}} \\
    &= \frac{1}{2\sigma^2} \frac{ (\boldsymbol{\Delta}^\top \hat{B}^{-1}\left(\boldsymbol{e}_r \right))^2}{ \|\hat{B}^{-1}\left(\boldsymbol{e}_r \right)\|^2_{\mathrm{Diag}(1/w_a)}}.
\end{align*}
This gives the following expression for the characteristic time \begin{align*}
    \chartime(\bmu)^{-1} = \max_{\bw \in \allocationset} \min_{\pol' \in \neighbors} \frac{1}{2\sigma^2} \frac{ \left(\boldsymbol{\Delta}^\top \hat{B}_{\optpol}^{-1}\boldsymbol{e}_{r'} \right)^2}{ \|\hat{B}_{\optpol}^{-1}\boldsymbol{e}_{r'}\|^2_{\mathrm{Diag}(1/w_a)}}.
\end{align*}
This formulation of the inverse characteristic time allows us to perceive it as a zero-sum $\max-\min$ game, where the max-player chooses an exploration allocation and the min-player swaps one of the active constraints, at the optimal policy, with one inactive constraint.

\textbf{Step 3 for Part (b): Bounds on characteristic time from perturbation in policies.} From Corollary~\ref{prop:lb_lower_bound} we have 
\begin{align*}
\frac{1}{2\sigma^2} \frac{\left(\bmu^\top \left(\optpol - \pol' \right) \right)^2}{\|\optpol - \pol'\|_2^2} &= \frac{1}{2\sigma^2} \frac{\left(\bmu^\top \left(\optpol - \pol' \right) - \mu^*  \mathbf{1}^\top \left(\optpol - \pol' \right) \right)^2}{\|\optpol - \pol'\|_2^2}\\
     &= \frac{1}{2\sigma^2} \frac{\left(\bmu -\mu^*  \mathbf{1}\right)^\top \left(\optpol - \pol' \right)^2}{\|\optpol - \pol'\|_2^2}\\
    &= \frac{1}{2\sigma^2} \frac{\left(\boldsymbol{\Delta}^\top \hat{B}^{-1}\boldsymbol{e}_r \right)^2}{\|\hat{B}^{-1}\boldsymbol{e}_r \|_2^2} \\
    &\leq \frac{\|\boldsymbol{\Delta}\|_2^2}{2\sigma^2} \frac{\sigma_{\max}^2(\hat{B})}{\sigma^2_{\min}(\hat{B})} \: . 
\end{align*}  

\textbf{Step 4 for Part (b): Concluding with complexity of bandit instance and constraints.} By referring to $\kappa(\hat{B}) \triangleq \frac{\sigma_{\max}(\hat{B})}{\sigma_{\min}(\hat{B})}$ as the condition number of $\hat{B}$, and $H \triangleq \frac{2\sigma^2}{\|\boldsymbol{\Delta}\|_2^2}$ as the quantifier complexity of bandit instance $\bmu$, we get
\begin{align*}
    \chartime(\bmu)^{-1} \leq \min_{\pol' \in \neighbors}  \frac{\kappa^2(\hat{B})}{H} \: .
\end{align*}

Hence, for any $\bmu$, we have that $\chartime(\bmu) \geq \frac{H}{\kappa^2}$, 
where $\kappa^2$ is the minimum condition number of any sub-matrix $\hat{B} \in \R^{K \times K}$ of $B$ consisting of $K$ linearly independent active constraints at $\optpol$. This leads to a lower bound 
\begin{align*}
    \E[\tau] \geq \Omega\left(\frac{H}{\kappa^2} \binaryKL{\delta}{1-\delta}\right).
\end{align*}

\subsection{Theorem~\ref{cor:projection} reduces to the standard BAI bounds with simplex constraints}
Recall the theorem statement:\\
If the arms follow Gaussian distributions with identical variance $\sigma^2$ and $w_a > 0$ $\forall a$, we have that the projection $\min_{\blambda \in \Domain: \blambda^\top(\optpol - \pol') \leq 0} \sum_{a=1}^K w_a \KL{\mu_a}{\lambda_a}$ for any $\pol' \in \mathcal{V}_{\F}({\optpol})$ is satisfied by 
\begin{align}\label{eq:lagrange_proj}
    \lambda_{a, \pol'} = \mu_a - \gamma \frac{(\optpol - \pol')_a}{w_a},
\end{align}
for $\gamma = \frac{\bmu^\top \left(\optpol - \pol' \right) }{\sum_a \frac{(\optpol - \pol')^2}{w_a}}$, and the characteristic time is
\begin{align*}
   \chartime(\bmu)^{-1} &= \max_{\bw \in \allocationset} \min_{\pol' \in \neighbors} \frac{1}{2\sigma^2} \frac{\left(\bmu^\top \left(\optpol - \pol' \right) \right)^2}{\sum_{a} \frac{1}{w_a}(\optpol - \pol')^2_a} \\
   &= \max_{\bw \in \allocationset} \min_{\pol' \in \neighbors} \frac{1}{2\sigma^2} \frac{ \|\optpol - \pol'\|^2_{\bmu\bmu^{\top}}}{ \|\optpol - \pol'\|^2_{\mathrm{Diag}(1/w_a)}}
\end{align*}
Here, $\mathrm{Diag}(1/w_a)$ is a diagonal matrix with $a$-th entry of the diagonal as $1/w_a$.

In the case of simplex constraints all extreme points corresponds to deterministic policies and we let $\pol_a$ corresponds to the policy that only plays arm $a$ and let $\pol^* = \pol_1$. For some $\pol_a$ we have, due to Equation~\eqref{eq:lagrange_proj}, \begin{align*}
    \lambda_{a', \pol_a} = \mu_{a'}, \forall a' \neq 1, a
\end{align*}
we further have $\gamma = \frac{\Delta_a}{\frac{1}{w_1} + \frac{1}{w_a}}$ and \begin{align*}
    \lambda_{1, \pol_a} &= \mu_1 - \frac{\mu_1 - \mu_a}{\frac{1}{w_1} + \frac{1}{w_a}}\frac{1}{w_1} = \mu_1 - w_a\frac{\mu_1 - \mu_a}{w_1 + w_a}  =
    \frac{1}{{w_1 + w_a} }\left(w_1\mu_1 + w_a \mu_a \right)
    \\
    \lambda_{a, \pol_a} &= \mu_a + \frac{\mu_1 - \mu_a}{\frac{1}{w_1} + \frac{1}{w_a}}\frac{1}{w_a} = \mu_a + w_1\frac{\mu_1 - \mu_a}{w_1 + w_a} = 
    \frac{1}{w_1 + w_a}\left(w_1\mu_1 + w_a\mu_a\right).
\end{align*}
Hence, $\lambda_{1, \pol_a}= \lambda_{a, \pol_a}$ and these are exactly the confusing instance one gets, for each arm $a$, in the BAI setting~\citep{Kaufmann16}. Plugging back into the expression for the characteristic time yields \begin{align*}
   \chartime(\bmu)^{-1} =  \max_w \min_a \frac{w_1w_a}{w_1 + w_a} \Delta_a^2.
\end{align*}

\newpage
\section{Upper bounds on sample complexity}\label{app:upper}

\subsection{Stopping criterion}

\begin{lemma}[\citet{Magureanu14}]\label{lm:mag}
$\forall \gamma > K + 1$ and $t \in \mathbb{N}$ it holds \begin{align*}
    P\left(\sum_{a=1}^K N_{a, t} \KL{\hat{\mu}_a}{\mu_a} \geq \gamma \right) \leq e^{-\gamma} \left(\frac{\lceil\gamma \log t\rceil \gamma}{K} \right)^K e^{K+1}
\end{align*}
\end{lemma}

The correctness of our stopping rule in Equation~\eqref{eq:stopping} follows easily from Lemma~\ref{lm:mag}. Let $\pol_\tau$ be our recommendation at stopping \begin{align*}
    P(\pol_\tau \neq \optpol) &\leq P\left(\exists t \in \mathbb{N}: \sum_{a=1}^K N_{a, t} \KL{\hat{\bmu}_{a, t}}{\bmu_a} \geq c(t, \delta)\right) \\
    &\leq  \sum_{t=1}^\infty e^{-c(t, \delta)} \left(\frac{\lceil c(t, \delta) \log t\rceil c(t, \delta)}{K} \right)^K e^{K+1}.
\end{align*}
We plug in $c(t, \delta)=\log \frac{t^\alpha C}{ \delta}$ and choose $C$ such that  \begin{align*}
    \sum_{t=1}^\infty \left(\frac{\lceil c(t, \delta) \log t\rceil c(t, \delta)}{K} \right)^K e^{K+1} \leq C
\end{align*} 
which yields \begin{align*}
    P(\pol_\tau \neq \optpol) \leq \delta.
\end{align*}

\newpage
\subsection{Upper bound for CTnS}\label{app:ctns}
\begin{proof}[Proof of Theorem~\ref{thm:ctns}] $ $\\
\textbf{Step 1: Defining Good Event.} Let $T \in \mathbb{N}$. For $\epsilon>0$ and $h(T)=\sqrt{T}$,  let $\event_T$ be the event $$\event_T \triangleq \bigcap_{t=h(T)}^T \left\{\|\hat{\bmu}_t - \bmu\|_\infty \leq \xi(\epsilon)\right\},$$ where $\xi(\epsilon) < \max_{\pol' \in \neighbors} \frac{1}{4\sqrt{K}}\bmu^\top \left(\optpol - \pol' \right) $ is such that \begin{align*}
    \|\bmu' - \bmu\|_\infty \leq \xi(\epsilon) \implies \sup_{\bw' \in w^*(\bmu')} \sup_{\bw \in w^*(\bmu)} \|\bw' - \bw \| \leq \epsilon
\end{align*}
 This $\xi(\epsilon)$ exists due to the upper hemicontinuity of $w^*(\bmu)$, Theorem~\ref{lm:properties}. 

\textbf{Step 2: Concentrating to Good Event}. We will make use of the following Lemma from \citet{garivier2016optimal} which bounds the probability of the compliment $\event_T^c$.
\begin{lemma}[Concentration around means~\citep{garivier2016optimal}]
    There exist two constants $B, C$ such that \begin{align*}
        P(\event_T^c) \leq BT \exp{\left(-CT^{\frac{1}{8}}\right)}
    \end{align*}
\end{lemma}
This Lemma is due to the fact that C-tracking ensure that each arm has been played at least $\sqrt{t}$ times at each time $t$, see next Lemma.
\begin{lemma}[\citet{garivier2016optimal}]\label{lm:track1}
    For all $t\geq 1$ and $\forall a$, C-Tracking ensures $N_{a, t} \geq \sqrt{t + K^2}- K$ and \begin{align}\label{eq:tracking}
        \max_{a} \left|N_{a, t} - \sum_{s=1}^{t} \bw_{a, s} \right| \leq K(1 + \sqrt{t})
    \end{align}
\end{lemma}

We now leverage to following tracking Lemma of \citet{Degenne19} which holds whenever we are tracking a set of optimal weights. 

\begin{lemma}[Concentration in allocations~\citep{Degenne19}]\label{lm:tracking2}
   Under $\event_T$, there exists a $T_\epsilon$ such that for $T$ where $h(T)\geq T_\epsilon$ C-tracking will satisfy \begin{align*}
       \inf_{\bw \in w^*(\bmu)} \| \frac{N_t}{t} - \bw\|_\infty \leq 3\epsilon, \forall t \geq 4\frac{K^2}{\epsilon^2} + 3\frac{h(T)}{\epsilon}
   \end{align*}
\end{lemma}

This shows that C-tracking is eventually going to produce an empirical distribution of plays that is close to an optimal allocation and the empirical distribution will converge to a point in $w^*(\bmu)$ as $t\rightarrow \infty$. We need Lemma~\ref{lm:tracking2} instead of the original tracking result in \citet{garivier2016optimal} since the optimal allocation does not need to be unique. However, we know from Theorem~\ref{lm:properties} that the set of optimal allocations $w^*(\bmu)$ is convex and we can thus apply Lemma~\ref{lm:tracking2}.

There exists a $T_\epsilon$ such that under $\event_T$ and $t \geq \max(T_\epsilon, h(T))$ we have \begin{align*}
     |(\bmu - \hat{\bmu}_t)^\top\optpol| \leq \sqrt{K}\xi < \frac{1}{4} \max_{\pol' \in \neighbors} \bmu^\top \left(\optpol - \pol' \right)
 \end{align*}
which implies that $\optpol = \argmax_{\pol \in \F} \hat{\bmu}_t^\top \pol$. This ensures that we will be computing the stopping criterion w.r.t. to the correct Alt-set $\Alt(\bmu)$.

\textbf{Step 3: Complexity given the Good Event.} Assume $T \geq T_\epsilon$ and let \begin{align*}
    C_{\epsilon, \F}(\bmu) \triangleq \inf_{\substack{\bmu': \|\bmu' - \bmu\|_\infty \leq \xi(\epsilon) \\
    \bw': \|\bw' - \bw\|_\infty \leq 3\epsilon, \forall \bw \in w^*(\bmu)}} D(\bw', \bmu', \F).
\end{align*}
This $C_{\epsilon, \F}(\bmu)$ gives the worst-case characteristic time we might compute in the algorithm due to the fact that our estimates are not exact. 

Assuming $\event_T$,  Lemma~\ref{lm:tracking2} gives for $t \geq T_\epsilon$ \begin{align*}
    D(\boldsymbol{N}_t, \hat{\bmu}_t, \F) \geq t C_{\epsilon, \F}(\bmu).
\end{align*}

\textbf{Step 4: Bounding the Stopping Time for Good and Bad Events.} Let $\tau_\delta$ be the stopping time, then \begin{align*}
    \min(\tau_\delta, T) \leq \sqrt{T} + \sum_{t= T_\epsilon}^T \mathbb{I}_{\tau_\delta > t}
\end{align*}
and plugging in our stopping rule, i.e. $D(\boldsymbol{N}_t, \hat{\bmu}_t, \F) > c(t, \delta)$ yields \begin{align*}
    T_\epsilon+ \sum_{t= T_\epsilon}^T \mathbb{I}(D(\boldsymbol{N}_t, \hat{\bmu}_t, \F) \leq c(t, \delta))  &\leq \sqrt{T}+ \sum_{t= T_\epsilon}^T \mathbb{I}(tC_{\epsilon, \F}(\bmu)  \leq c(T, \delta)) \\
    &\leq\sqrt{T} + \frac{c(T, \delta)}{C_{\epsilon, \F}(\bmu)}.
\end{align*}
We define $T_\delta := \inf\left\{T \in \mathbb{N}: \sqrt{T} + \frac{c(T, \delta)}{C_{\epsilon, \F}(\bmu)} \leq T \right\}$. Hence, \begin{align*}\label{eq:bound1}
    \E[\tau_\delta]  \leq T_\epsilon + T_\delta + \sum_{T=1}^\infty BT \exp{\left(-CT^{\frac{1}{8}}\right)} \leq T_\epsilon + T_\delta  + T'
\end{align*}
where $\sum_{t=1}^\infty BT \exp{\left(-Ct^{\frac{1}{8}}\right)} \leq T' < \infty$.
We bound $T_\delta$ in the same way as \citet{garivier2016optimal}. Let $C(\eta) = \inf\{T: T- \sqrt{T}\geq  T \frac{1}{1 + \eta}\}$ for some $\eta > 0$. Then \begin{align*}
    T_\delta \leq C(\eta) + \inf\left\{T \in \mathbb{N}:T\frac{C_{\epsilon, \F}(\bmu)}{1 + \eta} \geq c(T, \delta)\right\}.
\end{align*}

\textbf{Step 5: Obtaining the Asymptotic Bound.} Dividing Equation \ref{eq:bound1} with $\log \frac{1}{\delta}$ and taking the limit yields \begin{align*}
    \lim_{\delta \rightarrow 0} \inf \frac{\E[\tau_\delta]}{\log \frac{1}{\delta}} \leq \frac{\alpha (1 + \eta)}{C_{\epsilon, \F}(\bmu)}.
\end{align*}
$C_{\epsilon, \F}(\bmu)$ is continuous due to Theorem~\ref{lm:properties} and taking the limits $\eta, \epsilon \rightarrow 0$ yields \begin{align*}
        \lim_{\delta \rightarrow 0} \inf \frac{\E[\tau_\delta]}{\log \frac{1}{\delta}} \leq \alpha T_\F(\bmu), \forall \alpha > 1.
\end{align*}
\end{proof}
\newpage
\subsection{Upper bound for CGE}\label{app:cge}
The proof follows the same structure  as the proof of Theorem 2 in \citet{Degenne} and we use the same concentration analysis. The main difference is that we have to adjust the definition of approximate optimistic saddle point algorithm.

\begin{proof}[Proof of Theorem~\ref{thm:cge}] $ $\\
\textbf{Step 1: Defining Good Event.}
We start by defining the good event \begin{align*}
    \event_T \triangleq \left\{ \forall t \leq T  \, \forall a,  N_{a, t} \KL{\hat{\mu}_{a, t}}{\mu_t} \leq f(t)\right\}
\end{align*}
where $f(t)=3 \log t + \log \log t $.

\textbf{Step 2: Concentration of Good Event.}

We can bound $\sum_{t=1}^\infty P(\event_T^c)$ using Lemma~\ref{lm:mag}. Hence, for any $t\in \mathbb{N}$ and arm $a$ \begin{align*}
    P( N_{a, t} \KL{\hat{\mu}_{a, t}}{\mu_t}  &\geq f(t)) \leq e^{-f(t)} (1 + f(t) \log t) f(t)\\
    &= \frac{e^2}{t^3 \log t}(f(t) + f(t)^2 \log t).
\end{align*}
Summing yields \begin{align}
    \sum_{a=1}^K\sum_{t=1}^\infty P(\event_T^c) \leq K + K\sum_{t=2}^\infty \frac{e^2}{t^3 \log t} \left(f(t) + f(t)^2 \log t \right) \leq KC < \infty.
\end{align}
Here a constant $C=21$ is sufficient. 

\textbf{Step 3: Starting from the Stopping Criterion}
The main idea of the proof is to work with the stopping criterion \begin{align*}
        c(t, \delta) &\geq \inf_{\lambda \in \Alt(\hat{\bmu}_t)} \sum_{a=1}^K N_{a, t} \KL{\hat{\mu}_{a, t}}{\lambda_a}
\end{align*}
and show that if we have the event $\event_T$, our current recommendation at some $t$ is the correct policy $\optpol$ and we haven't stopped yet, we can lower bound $c(t, \delta)$ in a way that depends on the characteristic time and properties of the no-regret learners. We start with assuming our current recommendation at some $t$ is the correct policy $\optpol$ and we have the event $\event_T$,   \begin{align*}
    c(t, \delta) \geq \inf_{\lambda \in \Alt(\hat{\bmu}_t)}  \sum_{s=1}^t \sum_{a=1}^K w_{a, s} \KL{\hat{\mu}_{a, t}}{\lambda_a} - (1 + \sqrt{t})K
\end{align*}
which follows from Tracking Lemma~\ref{lm:track1}. We now use a concentration result, originally in  Appendix D.1 of \citet{Degenne}, \begin{align}\label{eq:equation1}
        c(t, \delta) &\geq \inf_{\lambda \in \Alt(\hat{\bmu}_t)}  \sum_{s=1}^t \sum_{a=1}^K w_{a, s} \KL{\hat{\mu}_{a, s}}{\lambda_a} - (1 + \sqrt{t})K - O(\sqrt{t \log t}).
\end{align}
This steps follows from the Lipschitz property of the KL and the fact we have conditioned on $\event_T$ (see Step 8 for further details). Hence, \begin{align*}
    |\KL{\mu_a}{\lambda_a} - \KL{\hat{\mu}_{a, s}}{\lambda_a}| \leq L \sqrt{2\sigma^2 \frac{f(s)}{N_{a, s}}}
\end{align*}
which implies that \begin{align*}
   \sum_{s=1}^t \sum_{a=1}^K w_{a, s} \KL{\hat{\mu}_{a, t}}{\lambda_a} \geq \sum_{s=1}^t \sum_{a=1}^K w_{a, s} \KL{\mu_{a}}{\lambda_a} - L \sqrt{2 \sigma^2 K t f(t)}.
\end{align*}
Using the same result one more time yields \begin{align*}
    \sum_{s=1}^t \sum_{a=1}^K w_{a, s} \KL{\hat{\mu}_{a, t}}{\lambda_a} \geq \sum_{s=1}^t \sum_{a=1}^K w_{a, s} \KL{\hat{\mu}_{a, s}}{\lambda_a} - L \sqrt{2 \sigma^2 K t f(t)} - 2L \sqrt{2\sigma^2 f(t)}\left(K^2 + 2 \sqrt{2K t}  \right)
\end{align*}
which gives the result in Equation~\eqref{eq:equation1}.

\textbf{Step 4: Defining Approximate Optimistic Saddle Point under Constraints.}
We now introduce concepts and properties that will help us to further lower bound the RHS in Equation~\eqref{eq:equation1}. We extend the definition of an \emph{approximate optimistic saddle point algorithm} from \citet{Degenne} to the constraint setting.
\begin{definition}\label{def:sad}
An algorithm playing sequences of $(\bw_s, \blambda_s)_{s \leq t} \in \left(\allocationset \times \Alt \right)^{t}$ is said to be an \emph{approximate optimistic saddle point algorithm} with slack $x_t$ if \begin{align}\label{eq:sad}
    \inf_{\blambda \in \Alt(\bmu)} \sum_{s=1}^t \sum_{a=1}^K w_{s, a} \KL{\hat{\mu}_{a, s}}{\lambda_a} \geq \max_{\bw \in \allocationset} \sum_{a=1}^K \sum_{s=1}^t w_a U_{a, s} - x_t,
\end{align}
where $x_t$ is defined in Eq.~\eqref{eq:x} and the confidence bound 
\begin{align*}
    U_{a, s} = \max \left\{\frac{f(t)}{N_{a, s}}, \max_{\xi \in [\alpha_{a, s}, \beta_{a, s}]} \KL{\xi}{, \lambda_{a, s}} \right\}.
\end{align*}
\end{definition}
The difference in Definition~\ref{def:sad} compared to the definition of an approximate optimistic saddle point algorithm in \citet{Degenne} is that we in Equation~\ref{eq:sad} take the maximum over $\allocationset$ and instead of arms as in \citet{Degenne}. 
This is due to the fact that maximum over arms might not be in the set of feasible exploration policies $\allocationset$.

\textbf{Step 5: Definition of Regret of the Two Players.}
We define the regret of the allocation player, i.e. AdaGrad, as \begin{align}\label{eq:66}
  R^{\bw}_t = \max_{\bw \in \allocationset } \sum_{s=1}^t \sum_{a=1}^K w_a U_{a, s} - \sum_{s=1}^t \sum_{a=1}^K w_{a, t} U_{a, s}
\end{align}
and note that AdaGrad has an regret scaling of $R_{\bw}^t \leq O(\sqrt{Qt})$ where $Q$ is an upper bound on the losses such that $Q \geq \max_{x, y \in [\mu_{\min}, \mu_{\max}]} \KL{x}{y}$. For the instance player we define the regret as \begin{align}
    R^{\blambda}_t =\sum_{s=1}^t \sum_{a=1}^K w_{a, s} \KL{\hat{\mu}_{a, s}}{\lambda_{a, s}} - \inf_{\blambda \in \Alt(\bmu)} \sum_{s=1}^t \sum_{a=1}^K w_{a, s} \KL{\hat{\mu}_{a, s}}{\lambda_a}
\end{align}
and note that $R_{\blambda}^t \leq 0$ since the instance player is performing a best-response against $\bw_s$ at each $s$.

\textbf{Step 6: CGE is an Approximate Optimistic Saddle Point Algorithm}
We now show that the CGE is an approximate optimistic saddle point algorithm. From the regret properties of $\blambda$ player we have \begin{align*}
    \inf_{\blambda \in \Alt{(\bmu)}} \sum_{s=1}^t \sum_{a=1}^K w_{s, a} \KL{\hat{\mu}_{a, s}}{\lambda_a} \geq \inf_{\blambda \in \Alt{(\bmu)}} \sum_{s=1}^t \sum_{a=1}^K w_{s, a} \KL{\hat{\mu}_{a, s}}{\lambda_{a,s}}
\end{align*}
since $R_{\blambda}^t \leq 0$.

Let $C_{a, s} = U_{a, s} - \KL{\hat{\mu}_{a, s}}{\lambda_{a, s}}$. We have \begin{align*}
    \inf_{\blambda \in \Alt{(\bmu)}} \sum_{s=1}^t \sum_{a=1}^K w_{s, a} \KL{\hat{\mu}_{a, s}}{\lambda_a} \geq \inf_{\blambda \in \Alt{(\bmu)}} \sum_{s=1}^t \sum_{a=1}^K w_{s, a} U_{a, s}(\blambda) - \sum_{s=1}^t \sum_{a=1}^K w_{s, a}C_{a, s}.
\end{align*}

Now, we can combine Eq. ~\eqref{eq:equation1} and \eqref{eq:66} to get
\begin{align*}
        c(t, \delta) &\geq \inf_{\blambda \in \Alt{\bmu}} \sum_{s=1}^t \sum_{a=1}^K w_{s, a} U_{a, s}(\blambda) - \sum_{s=1}^t \sum_{a=1}^K w_{s, a}C_{a, s} - (1 + \sqrt{t})K - O(\sqrt{t \log t})
\end{align*}     
Now we use the properties of $R^{\bw}_t$ to get \begin{align*}
    \inf_{\blambda \in \Alt(\bmu)} \sum_{s=1}^t \sum_{a=1}^K w_{s, a} \KL{\hat{\mu}_{a, s}}{\lambda_s} \geq \max_{\bw \in \allocationset} \sum_{a=1}^K \sum_{s=1}^t w_a U_{a, s} - R_{\bw}^t - \sum_{s=1}^t \sum_{a=1}^K w_{a, s}C_{a, s}
\end{align*}
which shows that CGE is an approximate optimistic saddle point algorithm with slack \begin{align}\label{eq:x}
    x_t= R^{\bw}_t + \sum_{s=1}^t \sum_{a=1}^K w_{s, a}C_{a, t}.
\end{align}

\textbf{Step 7: Plug slack $x_t$ into Equation~\eqref{eq:equation1}.}
We now use the fact that CGE is an approximate optimistic saddle point algorithm in Equation~\eqref{eq:equation1} \begin{align}
    c(t, \delta) &\geq \max_{\bw \in \allocationset} \sum_{a=1}^K \sum_{s=1}^t w_a U_{a, s} - R^{\bw}_t - \sum_{s=1}^t \sum_{a=1}^K w_{s, a}C_{a, t}  - (1 + \sqrt{t})K - O(\sqrt{t \log t})
\end{align}

\textbf{Step 8: Concentration of $\sum_{a=1}^K w_{s, a}C_{a, t}$}

Assume the event $\event_T$. We have \begin{align*}
    |\KL{\mu_a}{\lambda_a} - \KL{\hat{\mu}_{a, s}}{\lambda_a}| \leq L \KL{\hat{\mu}_{a, s}}{\mu_a}
\end{align*}
due to the Lipschitz property of the KL-divergence and under the event $\event_T$ we have \begin{align*}
    |\KL{\mu_a}{\lambda_a} - \KL{\hat{\mu}_{a, s}}{\lambda_a}| \leq  L \sqrt{2\sigma^2 \frac{f(s)}{N_{a, s}}}.
\end{align*}
This implies that \begin{align*}
   \sup_{\xi \in [\alpha_{a, s}, \beta_{a, s}]} U_{a, s} - \KL{\xi}{\lambda_{a, s}}  \leq \max \left\{2L \sqrt{2\sigma^2 \frac{f(s)}{N_{a, s}}}, \frac{f(s)}{N_{a, s}} \right\}
\end{align*}
since either $U_{a, s} = \max_{\xi \in [\alpha_{a, s}, \beta_{a, s}]} \KL{\xi}{, \lambda_{a, s}}$ and the above is equal to the width of the confidence interval, or $U_{a, s} = \frac{f(s)}{N_{a,s}}$ and the above is trivially bounded $\frac{f(s)}{N_{a,s}}$ since the KL divergence is non-negative. Hence, \begin{align*}
    \sum_{s=K+1}^t \sum_{a=1}^K w_{s, a}C_{a, s} &\leq \sum_{s=K+1}^t \sum_{a=1}^K w_{s, a} \left(\frac{f(s)}{N_{a,s}} + 2L \sqrt{2\sigma^2 \frac{f(s)}{N_{a, s}}} \right)  \\
    &\leq f(t)\sum_{s=K+1}^t\sum_{a=1}^K \frac{w_{s, a}}{N_{a,s}} + 2L\sqrt{2\sigma^2f(t)}\sum_{s=K+1}^t\sum_{a=1}^K \frac{w_{s, a}}{\sqrt{N_{a,s}}} \\
    &\leq f(t)\left(K^2 + 2K\log \frac{t}{K} \right) + 2L\sqrt{2\sigma^2f(t)}\left(K^2 + 2\sqrt{2Kt}\right)\\
    &\leq O(\sqrt{t \log t}).
\end{align*}

We have \begin{align*}
    c(t, \delta) &\geq \max_{\bw \in \allocationset} \sum_{a=1}^K \sum_{s=1}^t w_a U_{a, s} - R^{\bw}_t - O(\sqrt{t \log t})  - (1 + \sqrt{t})K - O(\sqrt{t \log t}).
\end{align*}

\textbf{Step 9: Optimism}

We now use the fact that $U_{a, s} \geq \KL{\mu_a}{\lambda_a}$ under the event $\event_T$. Hence, \begin{align*}
    c(t, \delta) &\geq \max_{\bw \in \allocationset} \sum_{a=1}^K \sum_{s=1}^t w_a \KL{\mu_a}{\lambda_{a, s}} - R^{\bw}_t - O(\sqrt{t \log t})  - (1 + \sqrt{t})K - O(\sqrt{t \log t}).
\end{align*}

\textbf{Step 10: Get the Characteristic Time}
We note that \begin{align*}
    \max_{\bw \in \allocationset} \sum_{a=1}^K \sum_{s=1}^t w_a \KL{\mu_a}{\lambda_{a, s}} &\geq t \inf_{\lambda \in \Alt(\bmu)} \max_{\bw \in \allocationset}\sum_{a=1}^K  w_a \KL{\mu_a}{\lambda_{a}} \\
    &\geq \max_{\bw \in \allocationset} \inf_{\lambda \in \Alt(\bmu)} \sum_{a=1}^K  w_a \KL{\mu_a}{\lambda_{a}} = t \chartime^{-1}(\bmu).
\end{align*}
Rearanging yields \begin{align*}
    t \leq \chartime(\bmu)c(t, \delta) + R^{\bw}_t + O(\sqrt{t \log t})
\end{align*}

 \textbf{Step 11: Current Recommendation is the Wrong Policy.}
 The above result is conditioned on the fact that our current recommendation is correct. We now bound the number of time steps where the current recommendation is wrong, using similar argument as in \citet{Degenne}.

We define the Chernoff information as $\chinf(x, y) \triangleq \inf_{u \in \Domain}: \KL{u}{x} + \KL{u}{y}$. Assumption 1 gives that there $\exists$ $\epsilon > 0$ such that $\forall \blambda \in \Alt(\bmu)$, $\exists a'$ such that $\chinf(\lambda_{a'}, \mu_{a'}) > \epsilon$.

Assume that $\optpol \neq \argmax_{\pol \in \F} \hat{\bmu}_t^\top \pol$, i.e. if we stop we would recommend the wrong policy. This implies that $\hat{\bmu}_t \in \Alt(\bmu)$ and $\chinf(\hat{\mu}_{a, t}, \mu_a) \geq \epsilon$ for some arm $a$. Under the good event $\event_T$ we have $N_{a, t}\KL{\hat{\mu}_{a, t}}{\mu_a} \leq f(t)$ which implies that $\frac{f(t)}{N_{a, t}} \geq \epsilon$, since $\chinf(\hat{\mu}_{a, t}, \mu_a) \leq \KL{\hat{\mu}_{a, t}}{\mu_a}$. 

Let $\pol_s \triangleq \argmax_{\pol \in \F} \hat{\bmu}_s^\top \pol$, let $n_{\pol'}(t)$ be the number of stages where $\pol_s = \pol'$. Our goal is to upper bound $n_{\pol'}(t)$ for all extreme points $\pol' \in \F$ such that $'\pol' \neq \optpol$. For any $\blambda$ such that $\pol' = \argmax_{\pol \in \F} \blambda^\top \pol$ we have that $\bmu \in \Alt(\blambda)$ which gives \begin{align*}
  \epsilon_t =   \sum_{s=1, \pol_s \neq \optpol}^t \sum_{a=1}^K w_{a, s}\KL{\hat{\mu}_{a, s}}{\mu_a} \geq \sum_{\pol'\neq \optpol} \inf_{\blambda: \pol' \neq \argmax_{\pol} \blambda^\top \pol} \sum_{s=1, \pol_s = \pol'}^t \sum_{a=1}^K w_{a, s} \KL{\hat{\mu}_{a, s}}{\lambda_a}.
\end{align*}

We use the fact that on the time steps where $\pol_s=\pol'$ CGE is a optimistic saddle point algorithm with slack $x = R_{n_{\pol'}(t)}^{\bw} + \sum_{s=1, \pol_s = \pol'}^t \sum_{a=1}^K w_{s, a}C_{a, t}$. Hence, \begin{align*}
    &\inf_{\blambda: \pol' \neq \argmax_{\pol } \blambda^\top \pol} \sum_{s=1, \pol_s = \pol'}^t \sum_{a=1}^K w_{a, s} \KL{\hat{\mu}_{a, s}}{\lambda_a} \geq \\
    &\max_{\pol \in \allocationset} \sum_{s=1, \pol_s = \pol'}^t \sum_{a=1}^K  w_a U_{a, s} - R_{n_{\pol'}(t)}^{\bw}  - \sum_{s=1, \pol_s = \pol'}^t \sum_{a=1}^K w_{a, s}C_{a, s}.
\end{align*}

Under the event $\event_T$, and $s \leq t$ such that $\pol_s = \pol'$ there is an arm $a_s$ such that $U_{a_s, s} \geq \epsilon$. This implies that the sum $\max_{\pol \in \allocationset} \sum_{s=1, \pol_s = \pol'}^t \sum_{a=1}^K  w_a U_{a, s}$ is increasing linearly in $n_{\pol'}(t)$ since it is at least $\epsilon n_{\pol'}(t)$ under the concentration event $\event_T$. Thus, \begin{align*}
    \inf_{\blambda: \pol' \neq \argmax_{\pol} \blambda^\top \pol} \sum_{s=1, \pol_s = \pol'}^t \sum_{a=1}^K w_{a, s} \KL{\hat{\mu}_{a, s}}{\lambda_a} \geq \epsilon n_{\pol'}(t) - R_{n_{\pol'}(t)}^{\bw} - \sum_{s=1, \pol_s = \pol'}^t \sum_{a=1}^K w_{a, s}C_{a, s}
\end{align*}
and we know that $R_{n_{\pol'}(t)}^{\bw} = O(\sqrt{Q n_{\pol'}(t)})$ and $\sum_{s=1, \pol_s = \pol'}^t \sum_{a=1}^K w_{a, s}C_{a, s} = O(\sqrt{n_{\pol'}(t) \log n_{\pol'}(t)})$. This shows that $\epsilon_T$ increases at least linear in $n_{\pol'}(t)$ and thus also linearly in the number of time steps for whitch  $\pol_s \neq \optpol$. However, we have \begin{align*}
    \epsilon_t &= \sum_{s=1, \pol_s \neq \optpol}^t \sum_{a=1}^K w_{a, s}\KL{\hat{\mu}_{a, s}}{\mu_a} \leq \sum_{s=1}^t \sum_{a=1}^K w_{a, s} \frac{f(s)}{N_{a, s}}\\
    &\leq f(t)(K^2 + 2 K \log \frac{t}{K}).
\end{align*}
This implies that the current recommendation $\pol_s = \argmax \hat{\bmu}_t^\top \pol$ differs from $\optpol$ at most $O(\sqrt{t \log t})$ number of times.

 \textbf{Step 12: Final Bound.}
We know from the concentration of $\event_T$ that the number of times the compliment happens is upper bounded by $CK$ where $C$ is some problem independent constant. Putting it all together, we get that     $\E[\tau] \leq T_0(\delta) + CK,$
    where $$T_0(\delta) := \max \left\{t \in \mathbb{N}: t \leq \chartime(\bmu)c(t, \delta) + O(\sqrt{tQ}) + O(\sqrt{t \log t}) \right\}.$$
\end{proof}

\newpage
\section{Finding $\epsilon$-good policies under linear constraints}\label{app:epsilon_good}
In some cases one might be more interested in finding a policy that is $\epsilon$-close to the optimal one, i.e. finding $\pol'$ such that $\bmu^\top(\pol^*_{\bmu}- \pol') \leq \epsilon$, since this might have a much smaller sample complexity compared to searching for the optimal policy, see for example \citep{Garivier21} and \citep{Kocak21}. Both CTnS and CGE can in principle be extended to this case by changing the definition of the Alt-set. Given an instance $\bmu$ let $\Omega_{\F, \epsilon}(\bmu) := \{\pol \in \mathcal{N}_\F : \mu^\top(\optpol - \pol) \leq \epsilon\}$ be the set of $\epsilon$-good policies where $\mathcal{N}_\F$ is the set of all extreme points in the polytope $\F$. For each $\pol \in \Omega_{\F, \epsilon}(\bmu)$ we get the following Alt-set
\begin{align*}
    \Lambda_{\F,\epsilon}(\bmu, \pol) := \left\{ \blambda: \blambda^\top \left(\pol_{\blambda}^* - \pol \right) > \epsilon \right\}.
\end{align*}
Hence, the sample complexity might be different depending on which near-optimal policy the learner is considering. To handle this we would have to augment CTnS and CGE with the ``sticky'' approach developed in \citep{Degenne19}, where the learner commits to a recommendation since otherwise the learner might oscillate between  near-optimal policies and a mixture of their optimal allocations might not be optimal since $w^*(\bmu)$ is no longer ensured to be convex. Furthermore, due to $\epsilon>0$ it is no longer sufficient to project onto the normal cone and a naive implementation would have to optimize over $|\mathcal{N}_\F|$ convex sets which might only be tractable for a small set of constraints and/or arms.

\newpage
\section{Additional experimental analysis}\label{app:add}
In Figure~\ref{fig:exp_bern1} and \ref{fig:exp_bern2} we present results for arms with Bernoulli distributions and in Figure~\ref{fig:exp_1} and \ref{fig:add_exp_2} we present additional results for arms with Gaussian distributions. CTnS and CGE outperforms the uniform baseline in all cases and are usually on par with or better than the learner that always sample according to the asymptotically optimal allocation. We also see that the algorithms tend to be close to the lower bound in all cases. An interesting observation, which we commented on already in the main text, is that there tend to be a larger difference between all sampling rules for end-of-time constraints compared to anytime constraints. This is due to the fact that anytime constraints can be very restrictive on which sampling allocations are allowed and there might not be less room for an adaptive learner.

In the case of arms with Bernoulli distributions we did not use a close-form projection, as for Gaussian distributions, and instead computed the projection numerically by minimizing the KL-divergence subject to $\blambda^\top(\optpol - \pol')=0$, which is a convex problem. We discuss the effect of this in Section~\ref{sec:running}
\begin{figure}[h]
    \centering
    \begin{tabular}{cc}
     \subfloat[End-of-time constraints]{\includegraphics[width=0.5\textwidth]{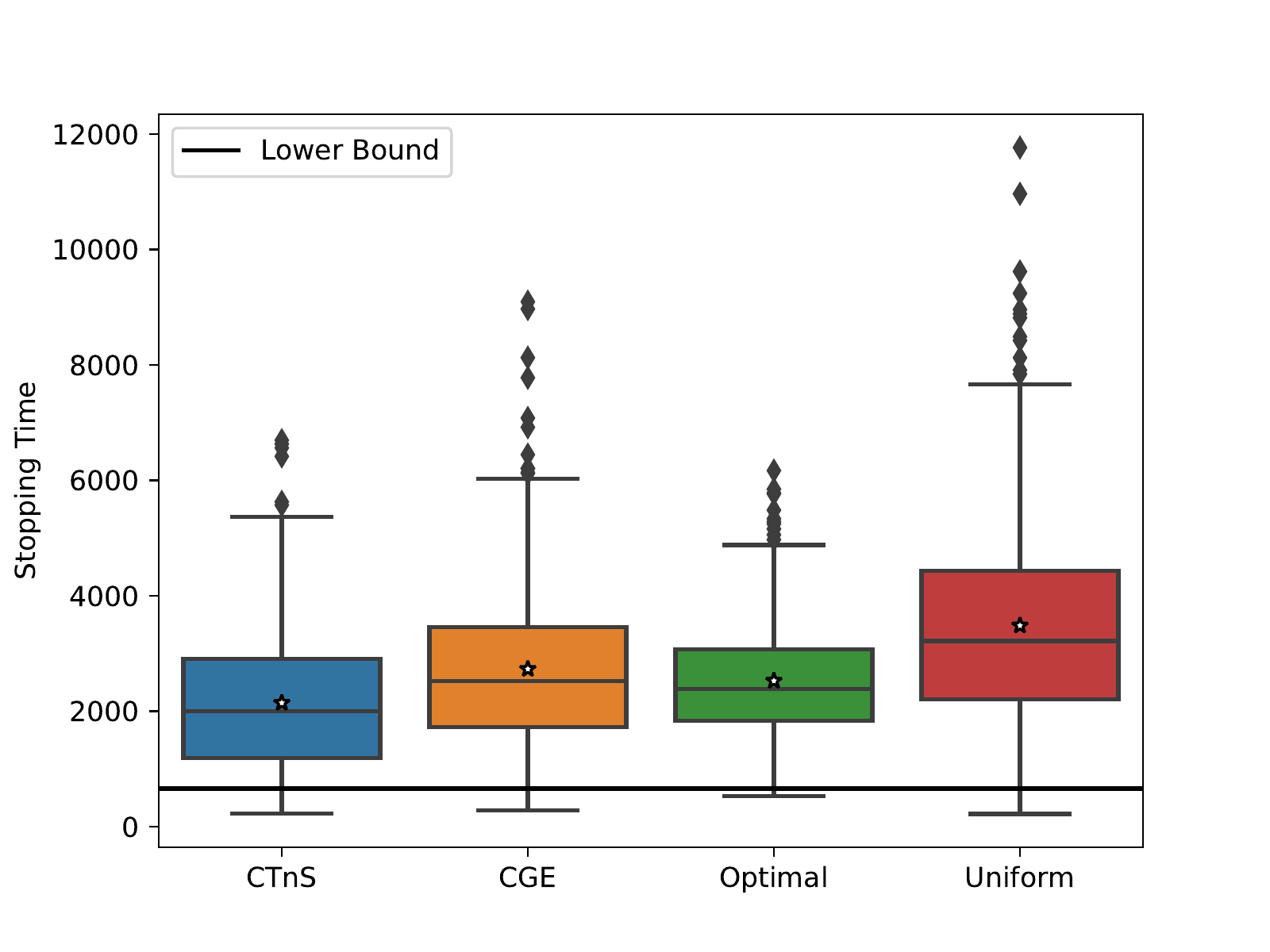}\label{fig:bexp}}   &
     \subfloat[Anytime constraints]{\includegraphics[width=0.5\textwidth]{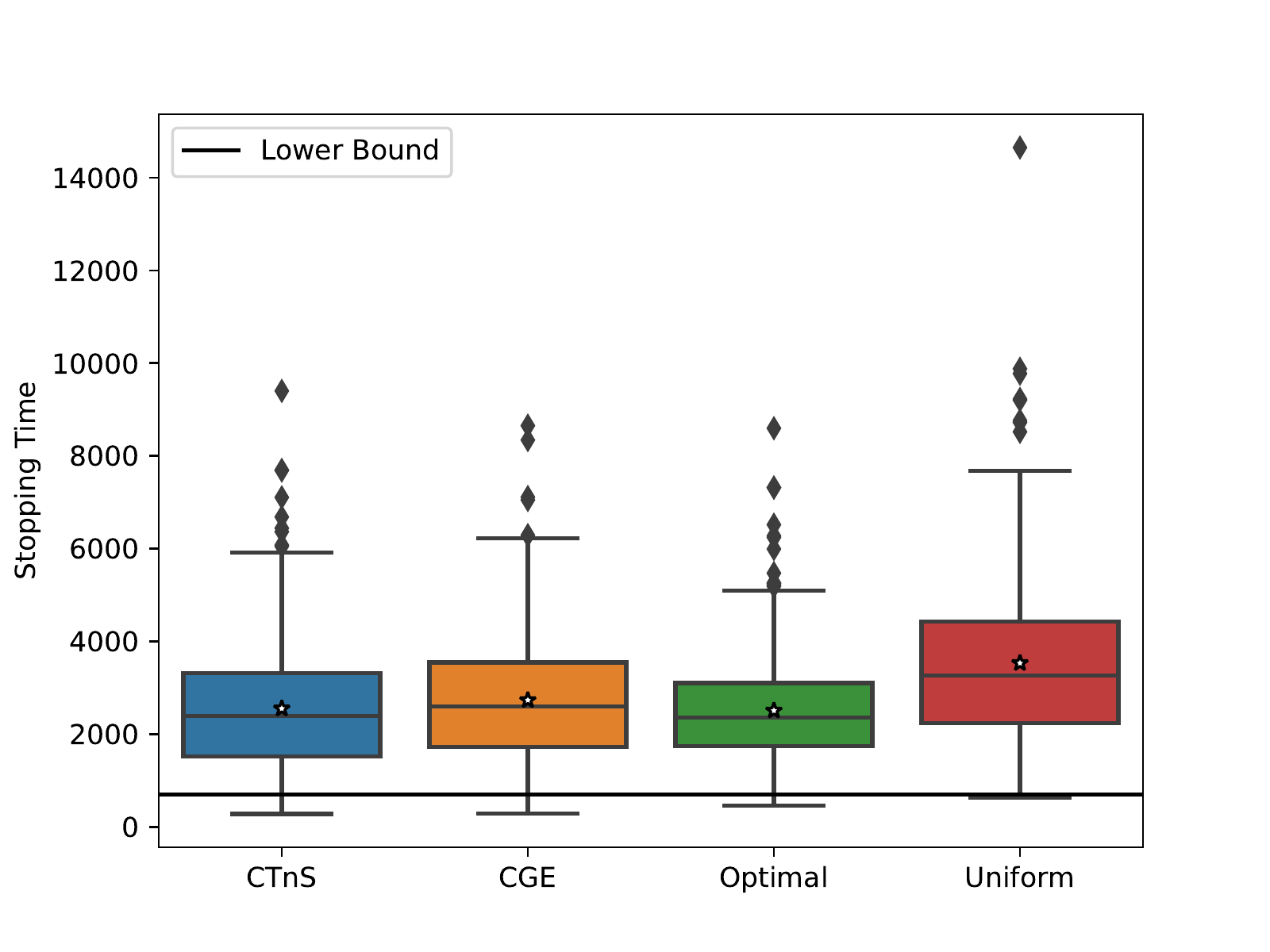}} 
        \end{tabular}
    \caption{End-of-time and Anytime constraints with \emph{Bernoulli} arms. The reward vector is $\bmu=(0.8, 0.7, 0.6, 0.5, 0.4, 0.3, 0.2)$ and the constraints are $\pol_1 + \pol_2 \leq 0.5$ and $\pol_3 +\pol_4 \leq 0.5$. Average over $500$ seeds and $\delta=0.1$. Optimal policy is $\pol_1=0.5$ and $\pol_3=0.5$.}
    \label{fig:exp_bern1}
\end{figure}

\begin{figure}[h]
    \centering
    \begin{tabular}{cc}
     \subfloat[End-of-time constraints]{\includegraphics[width=0.5\textwidth]{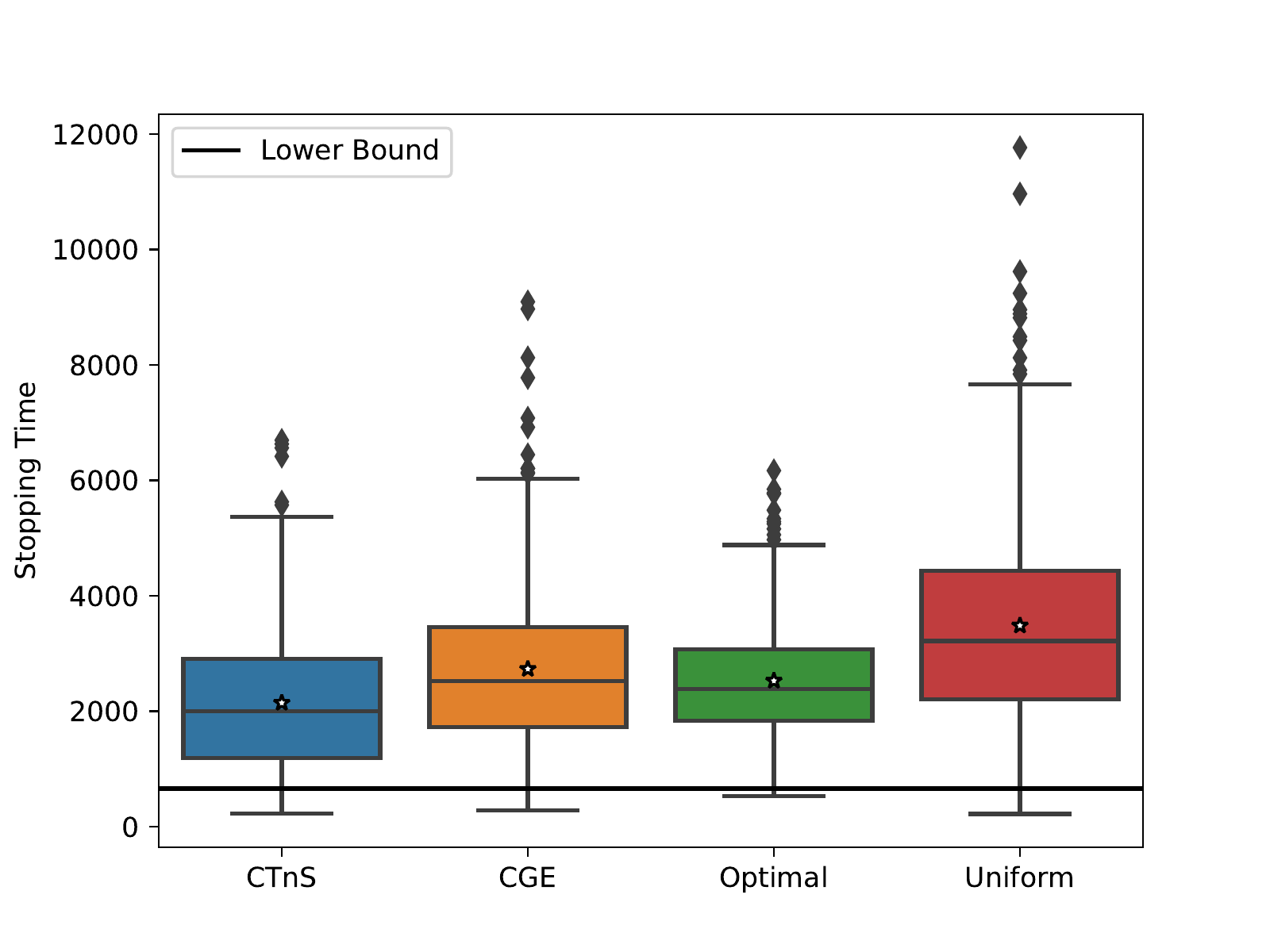}\label{fig:bexp2}}   &
     \subfloat[Anytime constraints]{\includegraphics[width=0.5\textwidth]{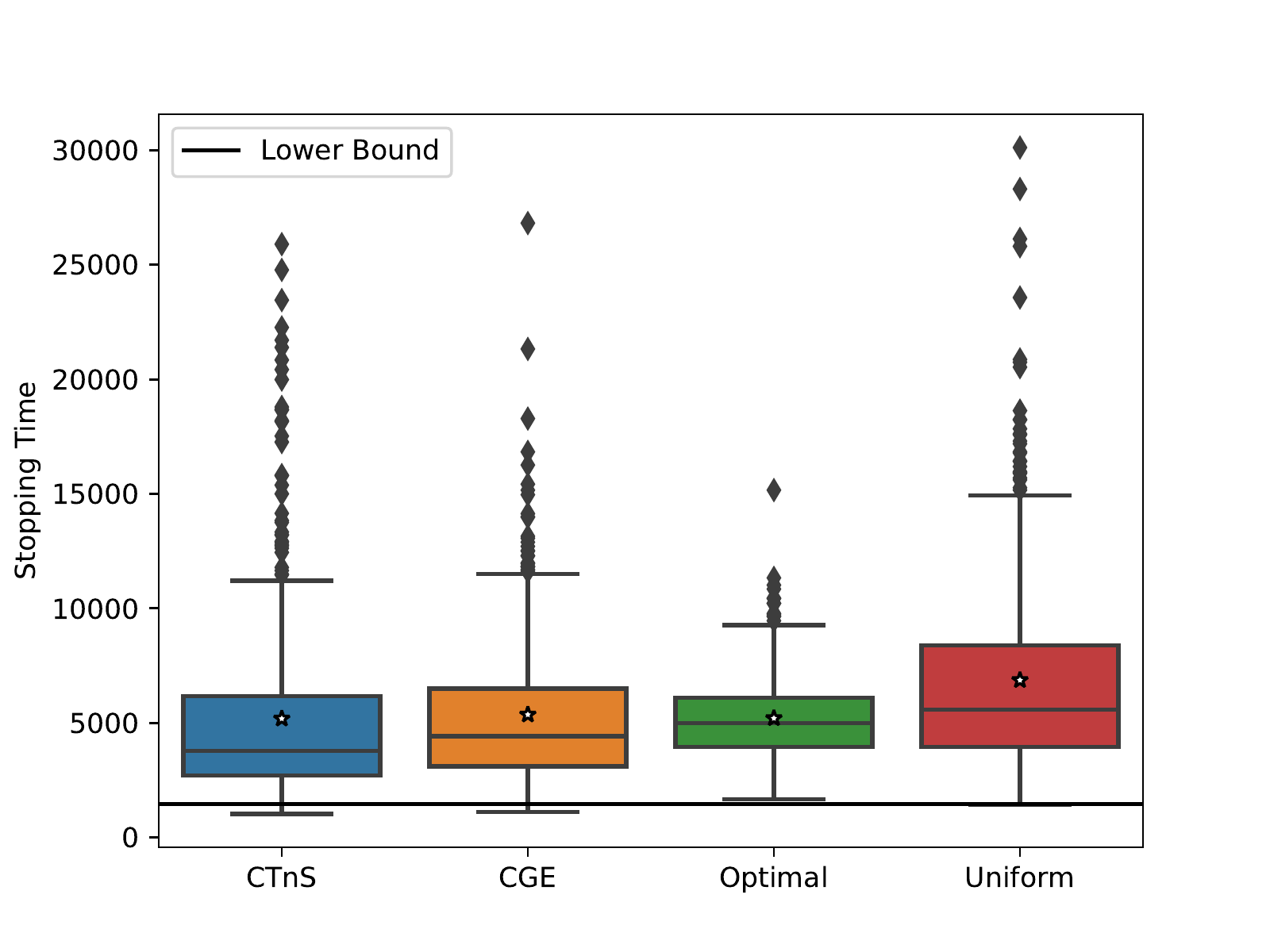}} 
        \end{tabular}
    \caption{End-of-time and Anytime constraints with Bernoulli arms. The reward vector is $\bmu=(0.8, 0.7, 0.6, 0.5, 0.4)$ and the constraints are $4\pol_1 - \pol_5 \leq 1$ and $3\pol_2 - \pol_4 \leq 1$. Average over $500$ seeds and $\delta=0.1$. Optimal policy is $\pol_1=0.25$, $\pol_2=0.33$ and $\pol_3=0.42$.}
    \label{fig:exp_bern2}
\end{figure}

\begin{figure}[h]
    \centering
    \begin{tabular}{cc}
     \subfloat[End-of-time constraints]{\includegraphics[width=0.5\textwidth]{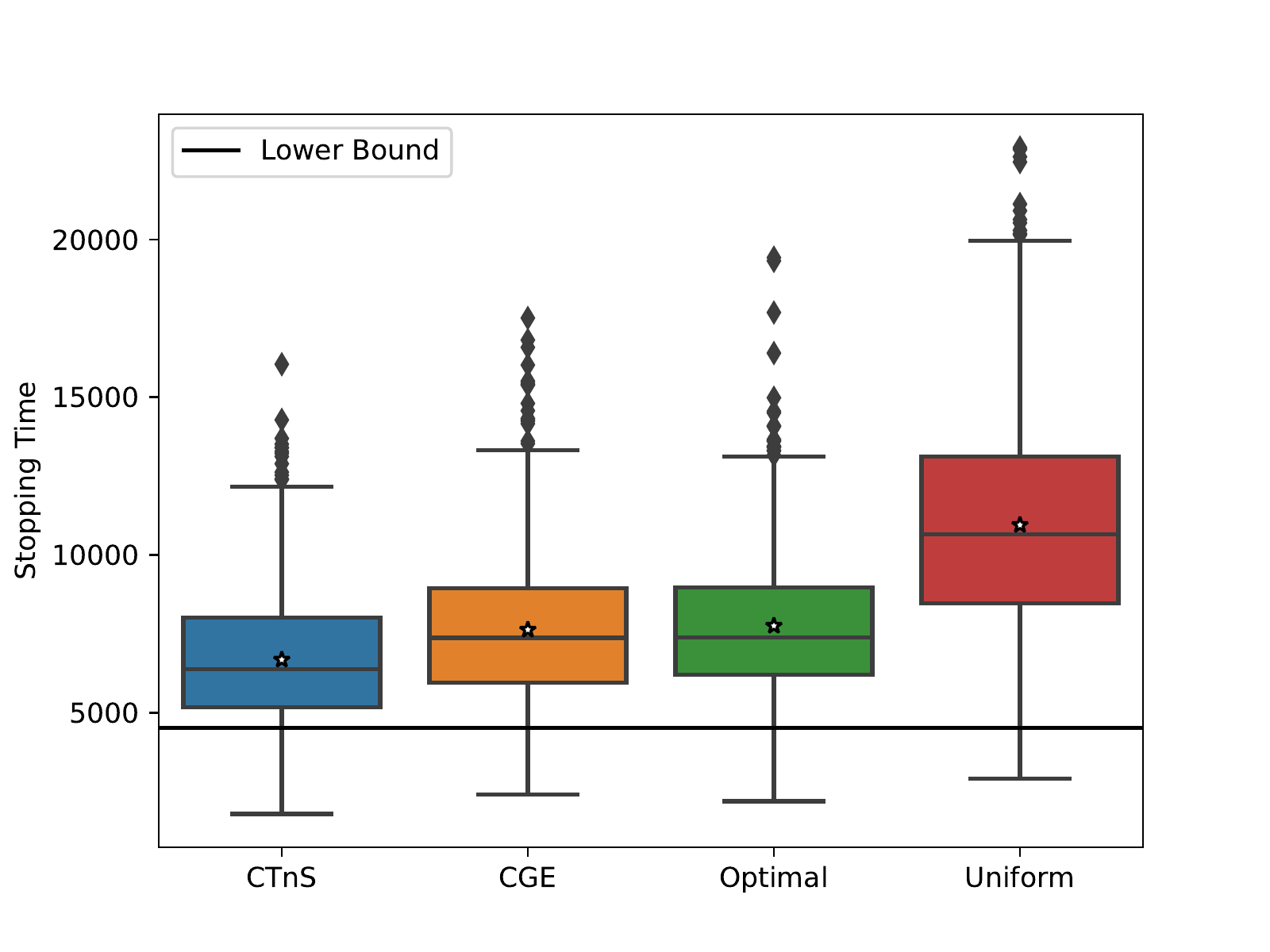}\label{fig:gexp}}    &
     \subfloat[Anytime constraints]{\includegraphics[width=0.5\textwidth] {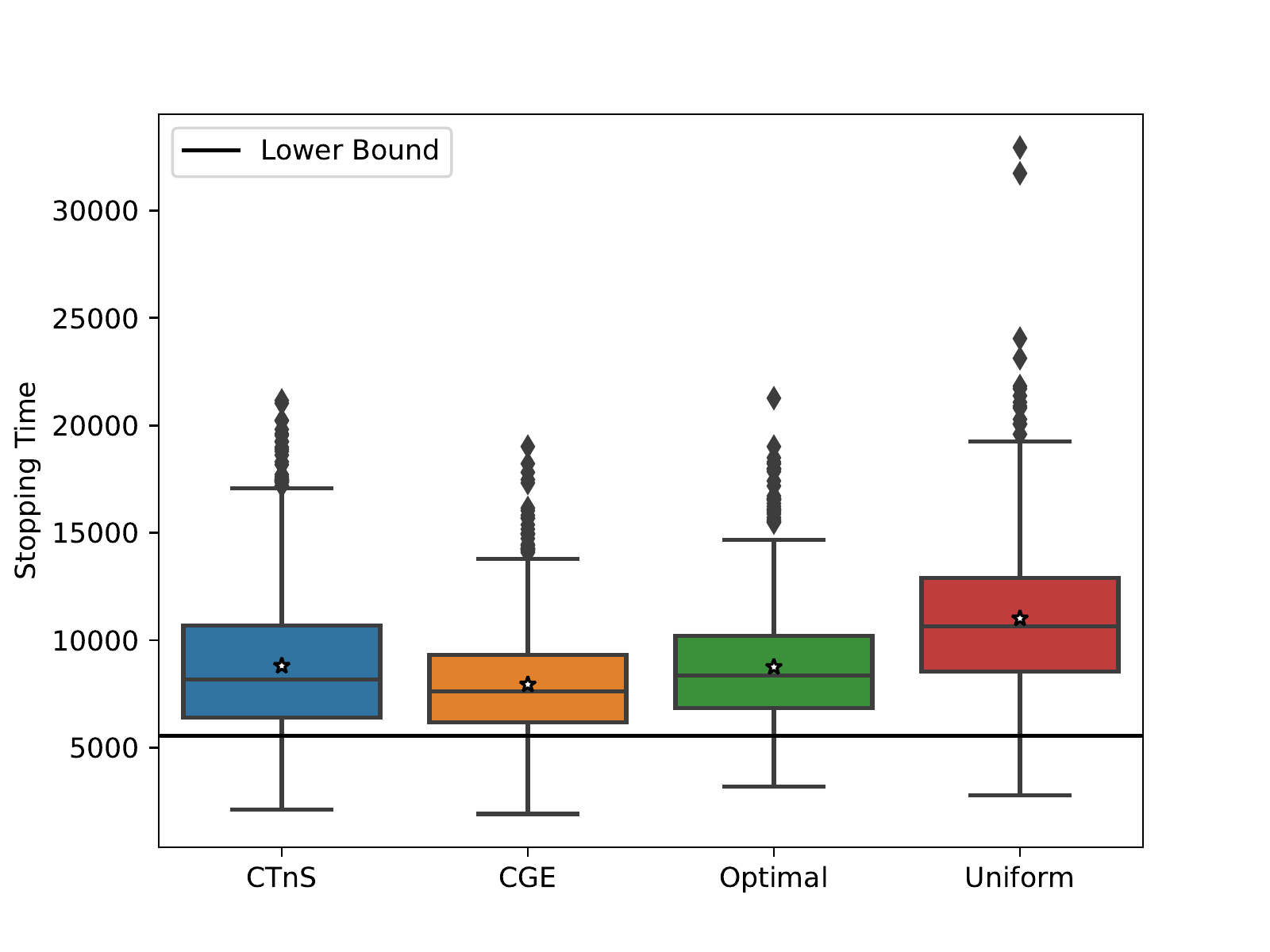}} 
        \end{tabular}
    \caption{End-of-time and Anytime constraints with \emph{Gaussian} arms $\sigma^2=1$. The reward vector is $\bmu=(2.0, 1.5, 1.45, 0.5, 0.3, -1.0, -1.0)$ and the constraints are $4\pol_1 + \pol_2 \leq 0.7$ and $\pol_2 +2\pol_3 \leq 0.5$. Average over $1000$ seeds and $\delta=10^{-4}$. Optimal policy is $\pol_1=0.05$, $\pol_2=0.5$ and $\pol_4=0.45.$}
    \label{fig:exp_1}
\end{figure}

\begin{figure}[h]
    \centering
    \begin{tabular}{cc}
     \subfloat[End-of-time constraints]{\includegraphics[width=0.5\textwidth]{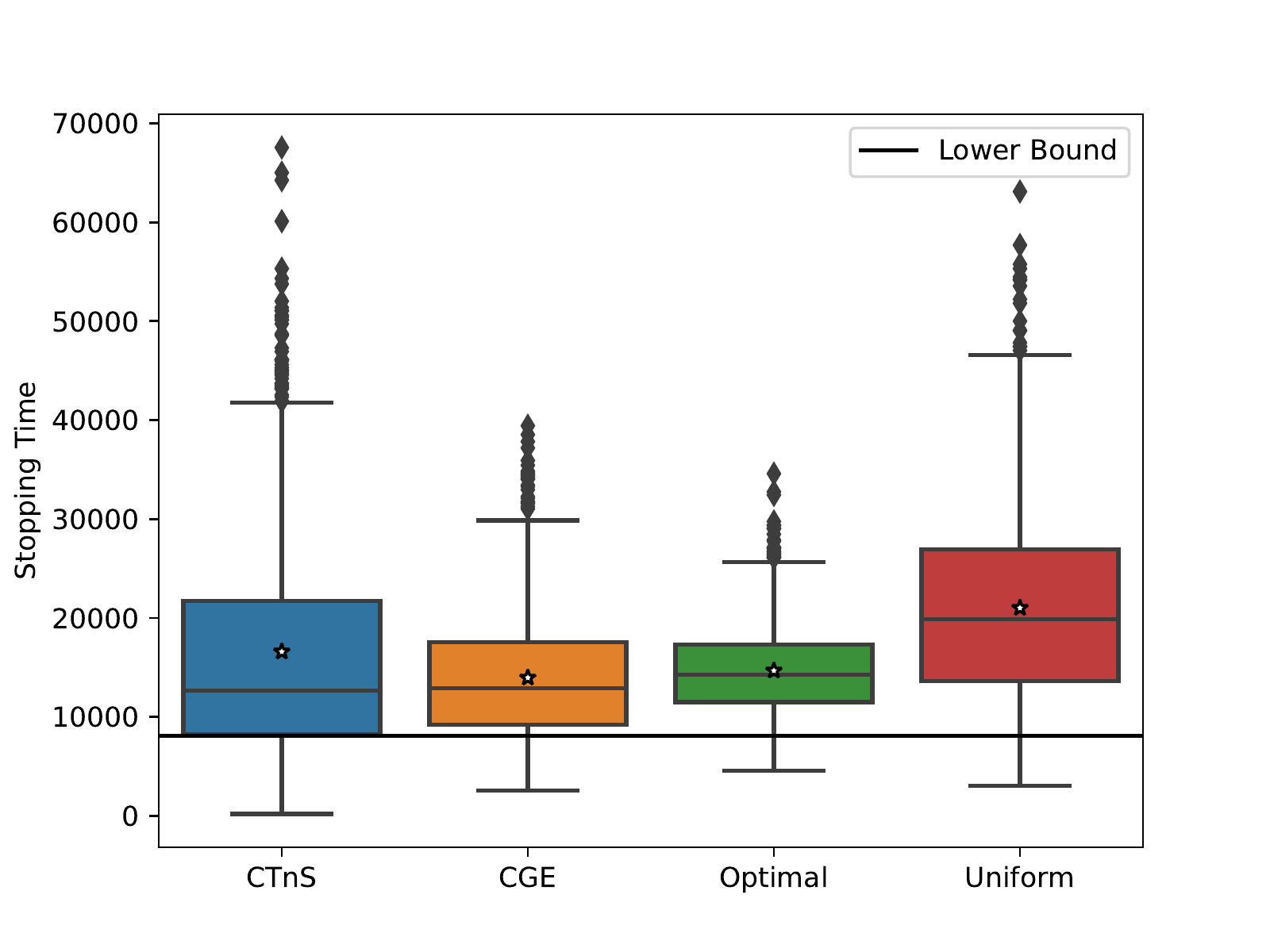}}    &
     \subfloat[Anytime constraints]{\includegraphics[width=0.5\textwidth]{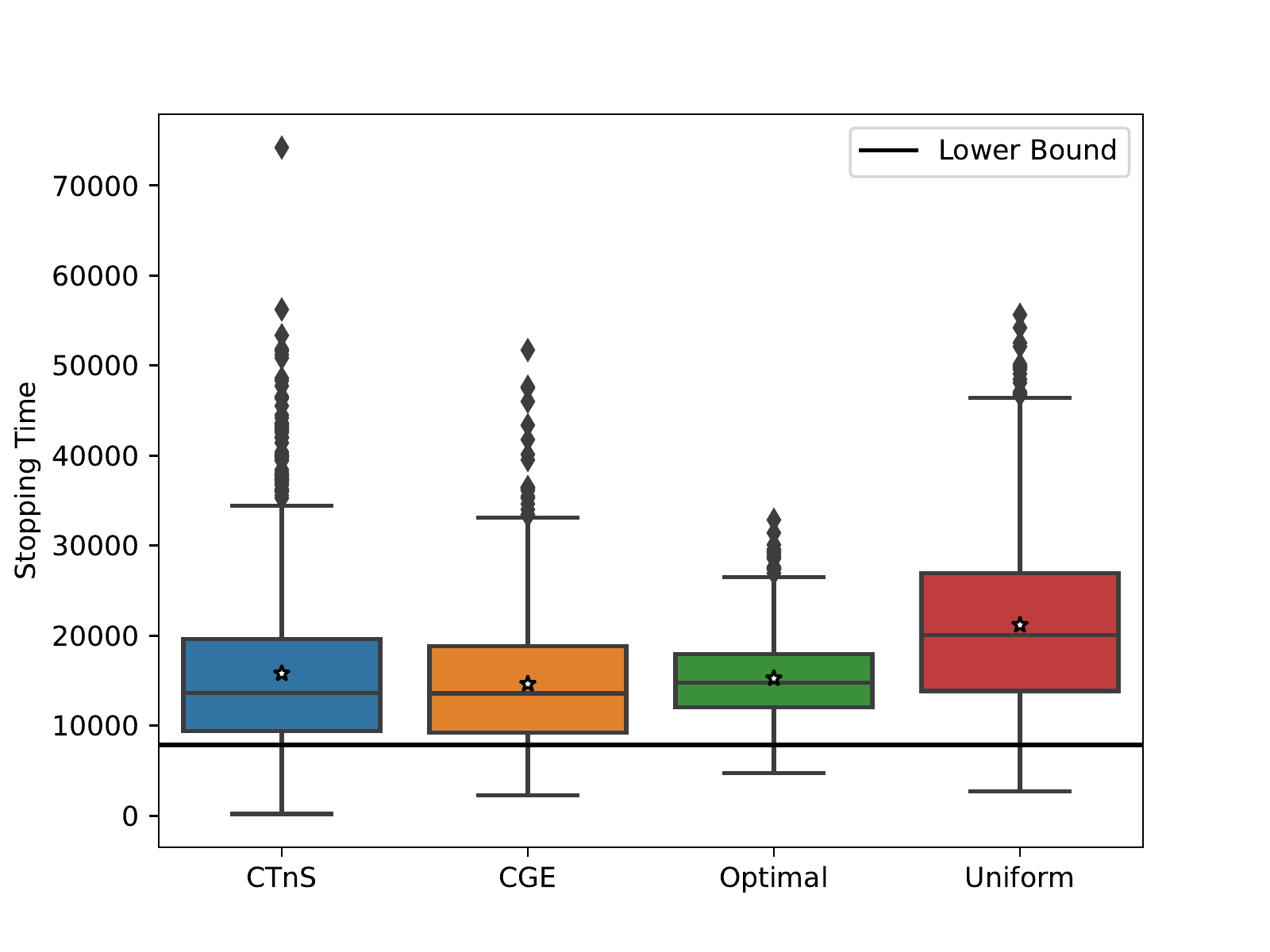}}
     \\
        \end{tabular}
    \caption{End-of-time and Anytime constraints with Gaussian arms $\sigma^2=1$. The reward vector is $\bmu=[1.0, 0.5, 0.4, 0.3, 0.2, 0.1]$ and the constraints are $\pol_1  -\pol_4 -\pol_5 -\pol_6 \leq 0.3$ and $\pol_2 \leq 0.7$. Average over $1000$ seeds and $\delta=10^{-3}.$ Optimal policy is $\pol_1 = 0.65$ and $\pol_4=0.35$.}
    \label{fig:add_exp_2}
\end{figure}
\newpage
\subsection{Running Times}\label{sec:running}
In Table~\ref{tab:running} we present the average time it take for the algorithms to check the stopping criterion and select a new arm to play. The test was performed on $1$ core of a Intel Xeon Gold 6130 CPU with $2.1$ GHz. Gaussian indicates the experiments in Figure~\ref{fig:gexp}, Bernoulli the experiments in Figure~\ref{fig:bexp} and IMDB the experiments in Figure~\ref{fig:imdb_end}. As expected CTnS is the algorithm requiring most computational time and the excessive running time it has on the experiment with Bernoulli distributions is due to the fact that we numerically solve the projection instead of relying on a close-form expression as in the case of Gaussian distributions. In contrast, we see that CGE has a relatively light computational footprint in all cases. Another advantage of CGE is that it performs a finite number of $\max$ calls at each iteration which can easily be parallelized for larger bandit instances with many constraints.  
\begin{table}[h!]
    \centering
    \begin{tabular}{c|c c c}
         Algorithm & Bernoulli & Gaussian & IMDB \\
         \hline
         CTnS & $1.00 \pm 0.244$ & $0.030 \pm 0.006$ & $0.033 \pm 0.015$ \\
         CGE & $0.02 \pm 0.001$ & $0.005 \pm 0.001$ & $0.008 \pm 0.001$ \\
         Uniform & $0.009 \pm 3 \times 10^{-4}$ & $0.001 \pm 1 \times 10^{-4}$ & $0.002 \pm 2 \times 10^{-4} $ \\
    \end{tabular}
    \caption{Average time, in seconds, it takes to check the stopping criterion and select a new arm for the different algorithms. The$\pm$ indicates one standard deviation. We omitted the optimal sampler since this one has the same running time as the uniform sampler.}
    \label{tab:running}
\end{table}
\newpage
\subsection{IMDB environment}
For resproducibility, here we provide the specifics of the IMDB data in the Table~\ref{tab:imdb} as used in the experiments (Figure~\ref{fig:exp_imdb}).

\begin{table}[h!]
    \centering
    \begin{tabular}{c|c c c c c}
        Movie & Average Rating & $\sigma$ & Action & Drama & Family  \\
        \hline
        The Net & $3.67$ & $1.26$ & $1$ & $1$ & $0$ \\
        Happily N'Ever After & $2.97$ & $1.30$ & $0$ & $0$ & $1$ \\
        Tomorrowland & $2.94$ & $1.31$ & $1$ & $0$ & $1$ \\
        American Hero & $3.52$ & $1.33$ & $1$ & $1$ & $0$ \\
        Das Boot & $3.18$ & $1.30$ & $0$ & $1$ & $0$ \\
        Final Destination 3 & $2.02$ & $0.93$ & $0$ & $0$ & $0$ \\
        Licence to Kill& $2.79$ & $1.22$ & $1$ & $0$ & $0$ \\
        The Hundred-Foot Journey& $2.97$ & $1.31$ & $0$ & $1$ & $0$ \\
        The Matrix& $2.32$ & $1.14$ & $1$ & $0$ & $0$ \\
        Creature& $2.53$ & $1.20$ & $0$ & $0$ & $0$ \\
        The Basket& $2.55$ & $1.19$ & $0$ & $1$ & $0$ \\
        Star Trek: The Motion Picture& $2.54$ & $1.16$ & $0$ & $0$ & $0$ \\
        \hline
    \end{tabular}
    \caption{Movies used in the experiments presented in Figure~\ref{fig:exp_imdb}. The optimal policy is $\optpol_1=0.3$, $\optpol_2=0.3$ and $\optpol_5=0.4$. We used the maximum $\sigma$ in the algorithms. This means that the algorithms didn't have access to the true $\sigma$ of each arm and instead modelled them all as Gaussian distributions with $\sigma=1.33$ but the rewards were sampled from the environment using the true $\sigma$.}
    \label{tab:imdb}
\end{table}

\clearpage
\section{Further discussion on the sub-optimality of PTnS}\label{app:ptns}
\begin{figure}[h]
    \centering
    \begin{tabular}{cc}
    \subfloat[Anytime Constraints]{\includegraphics[width=0.45\textwidth]{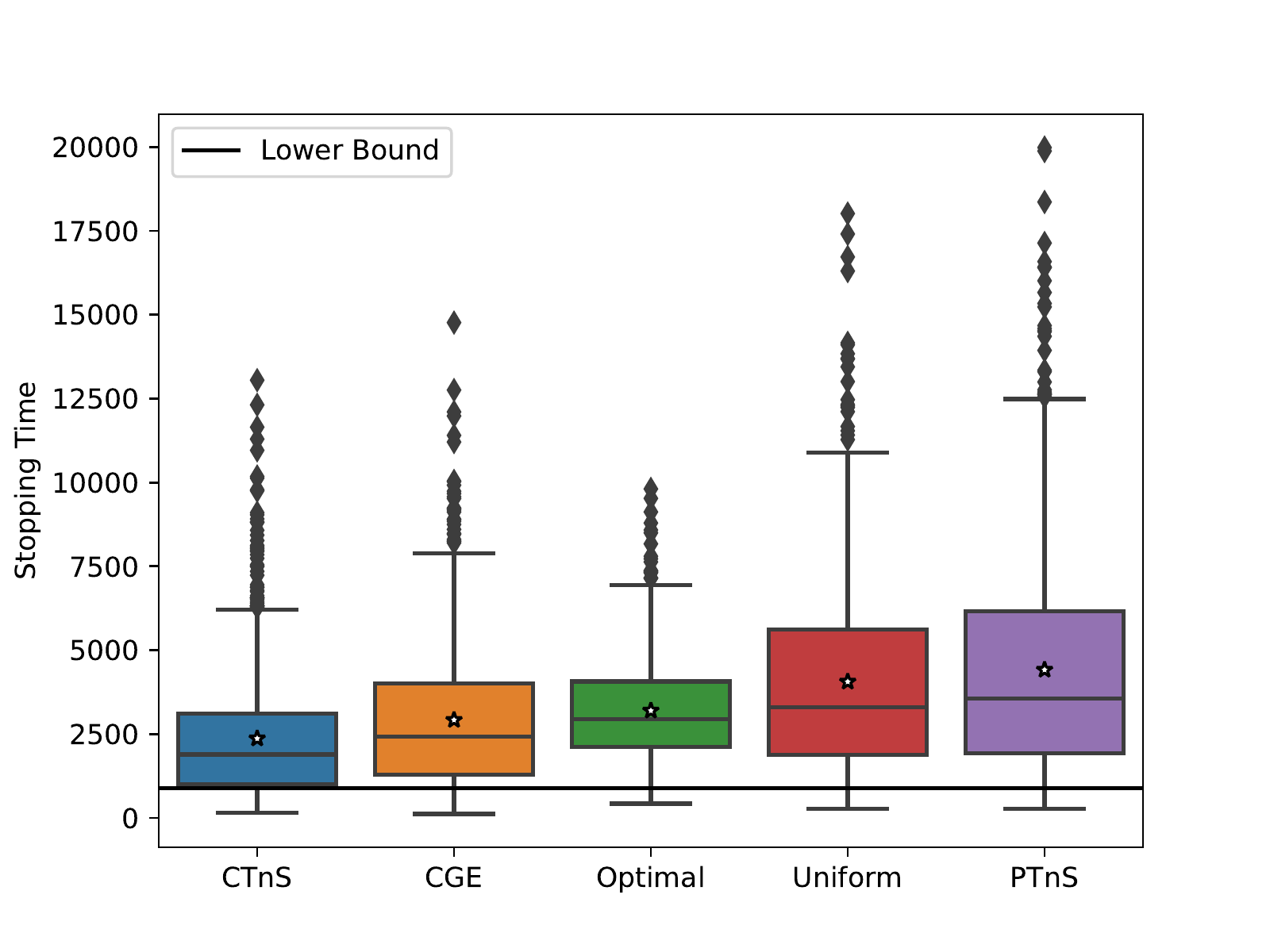}\label{fig:ptns_any_app}} & 
    \subfloat[End-of-Time Constraints]{\includegraphics[width=0.45\textwidth]{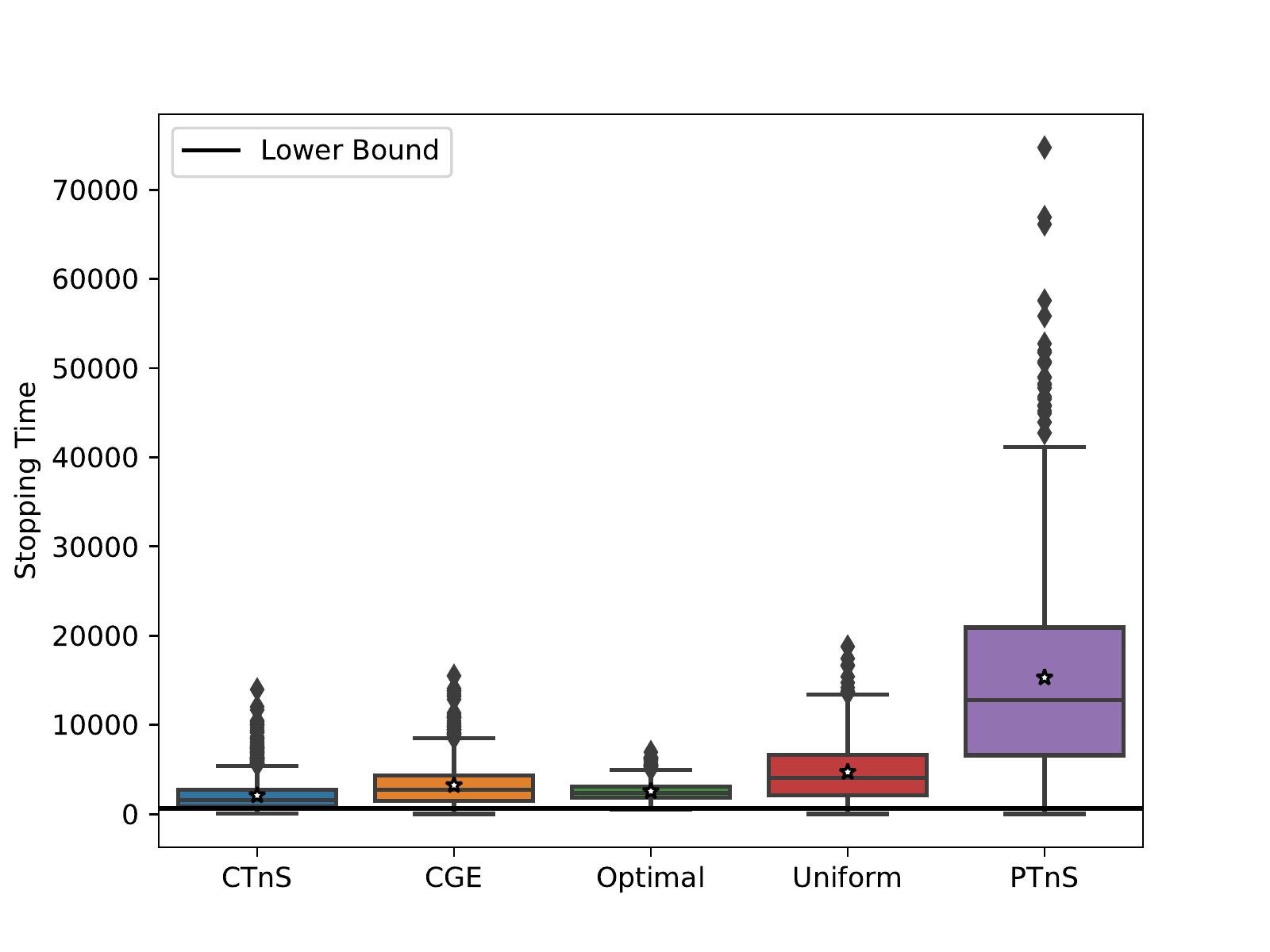}\label{fig:ptns_end_app}}
        \end{tabular}
    \caption{Problem instance with $8$ Gaussian arms with $\sigma=1$. The arm means are $\mu=[1.0, 0.7, 0.3, 0.0, -0.5, -1.0, -2.0, -3.0]$ and we have one constraint $7\pi_1 + 7\pi_2 + \pi_3 \leq 0.5$. The optimal policy is $\pi_3 = \pi_4=0.5$. Results for $\delta=0.1$ and $1000$ random seeds.}
    \label{fig:ptns_app}
\end{figure}

In Figure \ref{fig:ptns_app}, we consider an eight-armed bandit with Gaussian reward distributions with means 
\begin{align*}
   \bmu =  [1.0, 0.7, 0.3, 0.0, -0.5, -1.0, -2.0, -3.0],
\end{align*} 
variance $1$, and the constraint $7\pi_1 + 7\pi_2 + \pi_3 \leq 0.5$. 

We observe that PTnS performs the worst on this instance, specially in the end-of-time setting. This reflects the fact that \textit{the optimal allocation w.r.t. classical BAI bound does not have to be close to the optimal allocation given by the constraint version of the lower bound}. 

In Figure~\ref{fig:ptns_end_app}, the optimal allocation for the constraint problem is \begin{align*}
   \bw^* = [0.09, 0.02, { \bf 0.43}, {\bf0.36}, 0.03, 0.02, 0.02, 0.02],
\end{align*}
while the unconstrained optimal BAI allocation with the same $\bmu$ is 
\begin{align*}
    \hat{\bw} = [{ \bf 0.43}, {\bf 0.42}, 0.05, 0.03, 0.02, 0.02, 0.02, 0.02].
\end{align*}
Hence, PTnS focuses on exploring arm $1$ and $2$ the most, which makes sense without any constraints.
In contrast, the optimal allocation under constraint, i.e. $w^*$, suggests that one should focus on arm $3$ and $4$ as the constraint puts a disproportional cost on arm $1$ and $2$. 

In the anytime scenario, Figure~\ref{fig:ptns_any_app}, the optimal allocation is 
\begin{align*}
    w^* = [0.02, 0.01, { \bf 0.32}, {\bf 0.54}, 0.03, 0.03, 0.03, 0.03].
\end{align*}
In this scenario, the allocation $\hat{\bw}$, computed by PTnS, is no longer feasible and PTnS instead converges to the projected version
\begin{align*}
    w' = [0.03, 0.02, 0.12, 0.18, {\bf 0.16}, {\bf 0.16}, {\bf 0.16}, {\bf 0.16}].
\end{align*}
We observe that the previous issue is now mitigated by the projection, PTnS is no longer overly obsessed with arm $1$ and $2$. However, another issue arises as the projection distributes a substantial probability to the arms $5-8$, which are highly sub-optimal. These phenomena lead to worse performance of PTnS w.r.t. CTnS and CGE, as shown in Figure~\ref{fig:ptns_any_app}.

\clearpage
\section{Useful definitions and results}\label{app:useful}
\begin{definition}[Upper hemicontinuity]
We say that a set-valued function $C:\Theta \rightarrow \Omega$ is upper hemicontinuous at the point $\theta \in \Theta$ if for any open set $S \subset \Omega$ with $C(\theta) \in S$ there exists a neighborhood $U$ around $\theta$, such that $\forall x \in U$, $C(x)$ is a subset of S.
\end{definition}
\begin{theorem}[Berge's maximum theorem~\citep{berge1963topological}]\label{thm:berge}
    Let $X$ and $\Theta$ be topological spaces. Let $f:X \times \Theta \rightarrow \R$ be a continuous function and let $C: \Theta \rightarrow X$ be a compact-valued correspondence such that $C(\theta) \neq \emptyset$ $\forall \theta \in \Theta$.  
    If $C$ is continuous at $\theta$ then $f^*(\theta) = \sup_{x \in C(\theta)} f(x, \theta)$ is continuous and $C^*=\{x \in C(\theta): f(x, \theta) = f^*(\theta)\}$ is upper hemicontinuous.
\end{theorem}
Below we restate the upper bound on the sample complexity of the Gamified Explorer (GE) of \citet{Degenne}.
\begin{theorem}[Theorem 2 in \citet{Degenne}]
    The sample complexity of GE is \begin{align*}
        \E[\tau] \leq T_0(\delta) +\frac{eK}{a}
    \end{align*}
    where \begin{align*}
        T_0(\delta) = \max \{t \in \mathbb{N}: t \leq T(\bmu)c(t, \delta) + C_{\bmu}(R_t^{\blambda} + R_t^{\bw} + O(\sqrt{t \log t}))\}
    \end{align*}
    where $R_t^{\blambda}$ is the regret of the instance player, $R_t^{\bw}$ the regret of the allocation player and $C_{\bmu}$ an instance-dependent constant.
\end{theorem}
\end{document}